\documentclass[a4paper,10pt]{article}
\pdfoutput=1
\usepackage[utf8x]{inputenc}

\usepackage{mathtools}  
\mathtoolsset{showonlyrefs}

\usepackage{amssymb,latexsym,amsfonts,amsmath,amsthm}
\usepackage{graphicx}
\usepackage{color}
\usepackage{subfigure}
\usepackage{bm}
\newtheorem{theorem}{Theorem}[section]
\newtheorem{corollary}[theorem]{Corollary}
\newtheorem{lemma}[theorem]{Lemma}
\newtheorem{proposition}[theorem]{Proposition}

\usepackage{setspace}

\theoremstyle{definition}
\newtheorem{definition}[theorem]{Definition}
\newtheorem{remark}[theorem]{Remark}

\numberwithin{equation}{section}

\newcommand{\R}{\mathbb{R}}
\newcommand{\C}{\mathbb{C}}
 \newcommand{\N}{\mathbb{N}}
 \newcommand{\Z}{\mathbb{Z}}
 \newcommand{\scalprod}[1]{\left\langle #1  \right \rangle}

\usepackage[hmargin=2.5cm,vmargin=2.5cm]{geometry}
 
 \DeclareMathOperator*{\argmin}{arg min}

\graphicspath{{figures//}}

\title{
Stable and robust sampling strategies for compressive imaging}
\author{Felix Krahmer and Rachel Ward}

\begin{document}

\maketitle

\begin{abstract}
In many signal processing applications, one wishes to acquire images that are sparse in transform domains such as spatial
finite differences or wavelets using frequency domain samples. For such
applications, overwhelming empirical evidence suggests that superior
image reconstruction can be obtained through \emph{variable density}
sampling strategies that concentrate on lower frequencies.  The wavelet and Fourier transform domains are not incoherent because low-order wavelets and low-order frequencies are correlated, so compressive sensing theory does not immediately imply sampling strategies and reconstruction guarantees.   In this paper
we turn to a more refined notion of coherence -- the so-called \emph{local coherence} -- measuring for each sensing vector separately how correlated it is to the sparsity basis.
For Fourier measurements and Haar wavelet sparsity, the local coherence can be controlled and bounded explicitly,  so for matrices comprised of frequencies sampled from a suitable inverse square power-law
density, we can prove the restricted isometry property with near-optimal embedding dimensions.  Consequently, the variable-density sampling strategy we provide allows for image reconstructions that are stable to sparsity defects
and robust to measurement noise. Our results cover both reconstruction by $\ell_1$-minimization and by total variation minimization.  The local coherence framework developed in this paper should be of independent interest in sparse recovery problems more generally, as it implies that for optimal sparse recovery results, it suffices to have bounded \emph{average} coherence from sensing basis to sparsity basis -- as opposed to bounded \emph{maximal} coherence -- as long as the sampling strategy is adapted accordingly. 
\end{abstract}

\section{Introduction}
%{\color{cyan} One often encounters the problem of reconstructing an image from Fourier transform measurements. }
The measurement process in a wide range of imaging applications such as radar, sonar, astronomy, and computer tomography, can be modeled -- after appropriate approximation and discretization -- as taking samples from weighted discrete Fourier transforms \cite{fsy10}. Similarly, it is well known in the medical imaging literature that the measurements taken in Magnetic Resonance Imaging (MRI) are well modeled as Fourier coefficients of the desired image. Within all of these scenarios, one seeks strategies for taking frequency domain measurements so as to reduce the number of measurements without degrading the quality of image reconstruction. A central feature of natural images that can be exploited in this process is that they allow for approximately sparse representation in suitable bases or dictionaries.
% For example, 
%grayscale digital images are well approximated by piecewise constant functions which are sparse with respect to many wavelet systems, but also spatial finite differences.

 The theory of compressive sensing, as introduced in \cite{do06b, crt}, fits comfortably into this set-up:  its key observation is that signals which allow for a sparse or approximately sparse representation can be recovered from relatively few linear measurements via convex approximation, provided these measurements are sufficiently \emph{incoherent} with the basis in which the signal is sparse.
     
\subsection{Imaging with partial frequency measurements}
Much work in compressive sensing has focused on the setting of imaging with frequency domain measurements \cite{crt, CRT06:Stable, RV08:sparse} and, in particular, towards accelerating the MRI measurement process \cite{ldsp08, ldp07}.  For images that have sparse representation in the canonical basis, the incoherence between Fourier and canonical bases implies that uniformly subsampled discrete Fourier transform measurements can be used to achieve near-optimal oracle reconstruction bounds:  up to logarithmic factors in the discretization size,  any image can be approximated from $s$ such frequency measurements up to the error that would be incurred if the image were first acquired in full, and then compressed by setting all but the $s$ largest-magnitude pixels to zero \cite{RV08:sparse, ra09-1, CRT06:Stable, cata06}.   Compressive sensing recovery guarantees hold more generally subject to incoherence between sampling and sparsity transform domains.  Unfortunately, natural images are generally \emph{not} 
directly sparse in the 
standard basis, but rather with respect to transform domains more closely resembling wavelet bases.  As low-scale wavelets are highly correlated (coherent) with low frequencies, sampling theorems for compressive imaging with partial Fourier transform measurements have remained elusive.

A number of empirical studies, including the very first papers on compressive sensing MRI \cite{ldsp08, ldp07}, suggest that better image restoration is possible by subsampling frequency measurements from variable densities preferring low frequencies to high frequencies. In fact, variable-density MRI sampling had been proposed previously in a number of works outside the context of compressive sensing, although there did not seem to be a consensus on an optimal density \cite{PREMRI, sh, tn, many, gk}.  For a thorough empirical comparison of different densities from the compressive sensing perspective, see \cite{WA10}.

Guided by these observations, the authors of \cite{pvw11} estimated the coherence between each element of the sensing basis with elements of the sparsity basis separately as a means to derive optimal sampling strategies in the context of compressive sensing MRI.  In particular, they observe that incoherence-based results in compressive sensing imply exact recovery results for more general systems if one samples a row from the measurement basis proportionally to its squared maximal correlation with the sparsity-inducing basis.  For given problem dimensions, they find the optimal distribution as the solution to a convex problem.   In \cite{pmgtvvw12}, the approach of optimizing the sampling distribution is combined with premodulation by a chirp, which by itself is another measure to reduce the coherence \cite{ailib, kw10}.  A similar variable-density analysis already appeared in \cite{rw11} in the context of sampling strategies and reconstruction guarantees for functions with sparse orthogonal polynomial 
expansions and will also be the guiding strategy of this paper (cf. Section \ref{riplocal}). After the submission of this paper, the idea of variable-density sampling has been extended to the context of block sampling \cite{bbw13}, motivated by practical limitations of MRI hardware.

\subsection{Contributions of this paper}
In this paper we derive near-optimal reconstruction bounds for a particular type of variable-density subsampled discrete Fourier transform for both wavelet sparsity and gradient sparsity models.  More precisely, up to logarithmic factors in the discretization size,  any image can be approximated from $s$ such measurements up to the error that would be incurred if the wavelet transform or the gradient of the image, respectively, were first computed in full, and then compressed by setting all but the $s$ largest-magnitude coefficients to zero.  Note that the reconstruction results that have been derived for uniformly-subsampled frequency measurements \cite{crt, f12} only provide such guarantees for images which are exactly sparse.  

A major role in determining an appropriate sampling density will be played by the {\em local coherence} of the sensing basis with respect to the sparsity basis, as introduced in Section~\ref{riplocal}. 
Consequently, an important ingredient of our analysis is Theorem \ref{Bincoherent}, which provides frequency-dependent bounds on inner products between rows of the orthonormal discrete Fourier transform and rows of the orthonormal discrete Haar wavelet transform.  In particular, the maximal correlation between a \emph{fixed} row in the discrete Fourier transform and \emph{any} row of the discrete Haar wavelet transform decreases according to an inverse power law of the frequency, and decays sufficiently quickly that the sum of squared maximal correlations scales only logarithmically with the discretization size $N$.   This implies, according to the techniques used in \cite{rw11, rw112, bdwz}, that subsampling rows of the discrete Fourier matrix proportionally to the squared correlation results in a matrix that has the restricted isometry property of near-optimal order subject to appropriate rescaling of the rows.

For reconstruction,  total variation minimization 
\cite{rudin1992nonlinear,candes2002new,osher2003image,strong2003,chan2006total,chambolle2004} will be our algorithm of choice.  In the papers \cite{nw1} and \cite{nw2}, total variation minimization was shown to provide stable and robust image reconstruction provided that the sensing matrix is incoherent with the Haar wavelet basis. 
Following the approach of \cite{nw1}, we prove that from variable density frequency samples, total variation minimization can be used for stable image recovery guarantees. 
%
%Following its introduction in , total variation regularization in imaging is used often for image denoising, deblurring, impainting, and segmentation \cite{.  In this paper we focus on the application of total variation minimization to the setting of undersampled frequency measurements.

%We note that in the case of uniformly subsampled DFT measurements,  \cite{crt} showed exact recovery guarantees for total variation minimization by appealing to the shift theorem for discrete Fourier transforms and treating it as an $\ell_1$-minimization problem with respect to the image gradient.     In , these exact recovery guarantees were extended to stable reconstruction guarantees with respect to the \emph{gradient} of the image.  Using the discrete Poincare inequality, such  gradient-level reconstruction bounds imply image-level reconstruction bounds which fall short of the optimal rate by a factor of the square root of the sparsity level. This factor does not appear in the bounds we provide for variable-density frequency samples, at the expense of an additional logarithmic factor in the discretization dimension $N$.

\subsection{Outline}

The remainder of this paper is organized as follows.  Preliminary notation is introduced in Section \ref{preliminaries}.  The main results of this paper are contained in Section \ref{main}.  Section \ref{background} reviews compressive sensing theory and Section \ref{riplocal} presents recent results on sampling strategies for coherent systems.  The main results on the coherence between Fourier and Haar wavelet bases are provided in Section \ref{local}, and proofs of the main results are contained in Section \ref{proofs}. Section \ref{sec:numerics} illustrates our results by numerical examples. We conclude with a summary and a discussion of open problems in Section \ref{summary}.

{\section{Preliminaries}\label{preliminaries}}

\subsection{Notation}

In this paper, we consider discrete images, that is, $N \times N$ blocks of pixels, and represent them as discrete functions $f \in \C^{N \times N}$. We write $f(t_1, t_2)$ to denote any particular pixel value and for a positive integer $N$, we denote the set $\{1,2,\dots, N \}$ by $[N]$.  By $f_1 \circ f_2$ we denote the Hadamard product, i.e., the image resulting from pointwise products of the pixel values, $f_1 \circ f_2(t_1,t_2)=f_1 (t_1,t_2)f_2(t_1,t_2)$.  On the space of such images, the $\ell_p$ vector norm is given by $\| f \|_p = \big( \sum_{t_1, t_2} | f(t_1, t_2) |^p \big)^{1/p}, 1 \leq p < \infty$, and $\| f \|_{\infty} = \max_{(t_1, t_2)} | f(t_1, t_2) |$.  The inner product inducing the $\ell_2$ vector norm is $\scalprod{f, g} = \sum_{t_1, t_2} \bar{f}(t_1, t_2) g(t_1, t_2)$, where $\bar{z}$ denotes the complex conjugate of number $z \in \C$.
By an abuse of notation, the ``$\ell_0$-norm'' $\| f \|_0 = \#\{(t_1, t_2): f(t_1, t_2) \neq 0\}$
counts the number of non-zero entries of $f$. 

An image $f$ is called
$s$-sparse if $\| f \|_0 \leq s$.  The error of best $s$-term approximation
of an image  $f$ in $\ell_p$ is defined as

{\footnotesize
\[
\sigma_s(f)_p = \inf_{g : \|g\|_0 \leq s} \|f -g \|_p.
\]
}
Clearly, $\sigma_s(f)_p = 0$ if $f$ is $s$-sparse. Informally, $f$ is called compressible
if $\sigma_s(f)_1$ decays quickly as $s$ increases. 

For two nonnegative functions $f(t)$ and $g(t)$ on the real line, we write $f\gtrsim g$ (or $f\lesssim g$) if there exists a constant $C>0$ such that $f(t) \geq C g(t)$ (or $f(t)\leq C g(t)$, respectively) for all $t>0$. 

The discrete directional derivatives of $f \in \C^{N \times N}$ are defined pixel-wise as
{\footnotesize
\begin{align}
f_x \in \C^{N-1 \times N}, \quad \quad f_x(t_1, t_2) &= f(t_1 +1, t_2) - f(t_1, t_2)   \label{Xx} \\
f_y \in \C^{N \times N-1}, \quad \quad
f_y(t_1, t_2) &= f(t_1, t_2+1) - f(t_1, t_2)
  \end{align}
  }
The discrete gradient transform $\nabla: \C^{N \times N} \rightarrow \C^{N \times N \times 2}$ is defined
in terms of the directional derivatives via
{\footnotesize
\begin{equation}
\label{grad}
\nabla f (t_1, t_2) := \Big( f_x(t_1, t_2), \hspace{1mm} f_y(t_1, t_2) \Big),
\end{equation}
}
where the directional derivatives are extended to $N\times N$ by adding zero entries.
The \emph{total variation} semi-norm is the $\ell_1$ norm of the image gradient, 
{\footnotesize
\begin{equation}\label{eq:TV}
\| f \|_{TV} := \| \nabla f \|_1 = \sum_{t_1, t_2} \big( | f_x(t_1, t_2) | + | f_y(t_1, t_2) | \big).
\end{equation}
}
Here we note that our definition is the \textit{anisotropic} version of the total variation semi-norm.  The \textit{isotropic} total variation semi-norm becomes the sum of terms
{\footnotesize
$$
\big|f_x(t_1, t_2) + i f_y(t_1, t_2) \big| = \big( f_x(t_1, t_2)^2 + f_y(t_1, t_2)^2 \big)^{1/2}.
$$
}
The isotropic and anisotropic total variation semi-norms are thus equivalent up to a factor of $\sqrt{2}$.  

\subsection{Bases for sparse representation and measurements}

The Haar wavelet basis is a simple basis which allows for good sparse approximations of natural images.  We will work primarily in two dimensions, but first introduce the univariate Haar wavelet basis as it will nevertheless serve as a building block for higher dimensional bases.
\begin{definition}[Univariate Haar wavelet basis]
The univariate discrete Haar wavelet system is an orthonormal basis of $\C^{2^p}$ consisting of the constant function
$h^0(t) \equiv 2^{-p/2}$, the step function $h^{1}_{0,0}=h^1$ given by

{\footnotesize
\begin{equation}
 h^1(t) =  \left\{ \begin{array}{cc} 2^{-p/2}, & 1 \leq t \leq 2^{p-1}, \nonumber \\
- 2^{-p/2}, & 2^{p-1} < t \leq 2^{p},
\end{array} 
\right.
\end{equation}
}
along with the dyadic step functions 
{\footnotesize
\begin{align*}
  h^1_{n,\ell}(t) &= 2^{\frac{n}{2}}h^1(2^n t -2^p \ell)\\
  &= \begin{cases}
              2^{\frac{n-p}{2}} \quad&\text{for }\qquad \qquad \,  \ell 2^{p-n} \leq t< (\ell+\frac{1}{2}) 2^{p-n}\\
	      -2^{\frac{n-p}{2}} \quad&\text{for }\quad\  (\ell+\frac{1}{2}) 2^{p-n} \leq t < (\ell+1) 2^{p-n}\\
              0\quad &\text{else,}
	\end{cases}
\end{align*}
}
for $(n,\ell)\in\mathbb{Z}^2$ satisfying $0<n<p$ and $0\leq \ell <2^n$.
\end{definition}
To define the bivariate Haar wavelet basis of $\C^{2^p \times 2^p}$, we extend the univariate system by the window functions
{\footnotesize
\begin{align*}
  h^0_{n,\ell}(t) &= 2^{\frac{n}{2}}h^0(2^n t -2^p\ell)\\ &= \begin{cases}
              2^{\frac{n-p}{2}} \quad&\text{for }\quad   \ell 2^{p-n} \leq t< (\ell +1) 2^{p-n}\\
              0\quad &\text{else.}
	\end{cases}
\end{align*}
}
The bivariate Haar wavelet system can now be defined via tensor products of functions in the extended univariate system.
In order for the system to form an orthonormal basis of $\C^{2^p\times 2^p},$ only tensor products of univariate functions with the same scaling parameter $n$ are included. 

\begin{definition}[Bivariate Haar wavelet basis]
\label{bvhaar}
The bivariate Haar system of $\C^{2^p\times 2^p}$ consists of the constant function $h^{(0,0)}$ given by 
{\footnotesize
\begin{equation*}
 h^{(0,0)}(t_1,t_2) = h^0(t_1) h^0(t_2) \equiv 2^{-p}
\end{equation*}
}
and the functions $h^e_{n,\ell}$ with indices in the range $0\leq n <p$, $\ell=(\ell_1,\ell_2) \in \Z^2 \cap [0,2^n)^2$, and \\
$e = (e_1, e_2) \in  \big\{ \{0,1\}, \{1,0\}, \{1,1\} \big\}$ given by 
{\footnotesize
\begin{equation*}
 h^{e}_{n,\ell}(t_1,t_2) = h_{n,\ell_1}^{e_1}(t_1) h_{n,\ell_2}^{e_2}(t_2).
\end{equation*}
}
We denote by ${\cal H}$ the bivariate Haar transform $ f \rightarrow \big( \scalprod{f, h_{n, \ell}} \big)_{n, \ell}$ and, by a slight abuse of notation, also the unitary matrix representing this linear map. 
\end{definition}

\noindent We will also work with discrete Fourier measurements. 

\begin{definition}[Discrete Fourier basis]

Let $N = 2^p$. The one-dimensional discrete Fourier system is an orthonormal basis of $\C^N$ consisting of the vectors 

{\footnotesize
\begin{equation}
 \varphi_{k}(t) = \frac{1}{\sqrt{N}} e^{i 2\pi t k/N}, \quad -N/2+1 \leq t \leq N/2,
\end{equation}
}
indexed by discrete frequencies in the range  $-N/2 + 1 \leq k \leq N/2$.
The two-dimensional discrete Fourier basis of $\C^{N \times N}$ is just a tensor product of one-dimensional bases, namely
{\footnotesize
\begin{align}
\label{defFour}
 \varphi_{k_1, k_2}(t_1, t_2) = \frac{1}{N} &e^{i 2\pi ( t_1 k_1 +t_2 k_2) / N}, \nonumber\\ & -N/2+1 \leq t_1, t_2 \leq N/2,
\end{align}
}
indexed by discrete frequencies in the range $-N/2 + 1 \leq k_1, k_2 \leq N/2.$

We denote by ${\cal F}$ the two-dimensional discrete Fourier transform $ f \rightarrow \big( \scalprod{f, \varphi_{k_1, k_2}} \big)_{k_1,k_2}$ and, again, also the associated unitary matrix.  Finally,  we denote by ${\cal F}_{\Omega}$ its restriction to a set of frequencies $\Omega \subset [N]^2$.  

\end{definition}

{\section{Main results}\label{main}}

Our main results say that appropriate variable density subsampling of the discrete Fourier transform will with high probability produce a set of measurements admitting stable image reconstruction via total variation minimization or $\ell_1$-minimization.  

While our recovery guarantees are robust to measurement noise, our guarantees are based on a \emph{weighted} $\ell_2$-norm such that high-frequency measurements have higher sensitivity to noise. This noise model results from the proof; empirical studies, however, suggest that the more standard uniform noise model yields superior performance. We refer the reader to Section~\ref{sec:numerics} for details.
Our first result concerns stable recovery guarantees for total variation minimization.

\begin{theorem}
\label{thm1}
Fix integers $N = 2^p, m,$ and $s$ such that $s\gtrsim \log(N)$ and 

{\footnotesize
\begin{equation}
\label{m:tv}
m \gtrsim s \log^{3}(s)\log^5(N). 
\end{equation}
}
Select $m$ frequencies $\{ (\omega_1^j, \omega_2^j ) \}_{j=1}^m \subset \{-N/2+1, \dots, N/2\}^2$ i.i.d. according to

{\footnotesize
\begin{align}
\label{inverseD}
\emph{\text{Prob}} &\big[ (\omega_1^j, \omega_2^j) = (k_1, k_2) \big]  = C_N \min\left(C, \frac{1}{k_1^2 + k_2^2 }\right)\nonumber \\&=: \eta(k_1, k_2), \quad -N/2+1 \leq k_1, k_2 \leq N/2,
\end{align}
}
where $C$ is an absolute constant and $C_N$ is chosen such that $\eta$ is a probability distribution. \\
  Consider the weight vector $\rho = (\rho_j)_{j=1}^m$ with $\rho_j = (1/\eta(\omega^j_1, \omega^j_2))^{1/2}$, and assume that the noise vector $\xi = (\xi_j)_{j=1}^m$ satisfies $\|\rho \hspace{.5mm} \circ \hspace{.5mm}  \xi\|_2\leq  \varepsilon \sqrt{m} $, for some $\epsilon>0$.  Then with probability exceeding $1 - N^{-C\log^3(s)}$, the following holds for all images $f \in \C^{N \times N}$:
 
\noindent Given noisy partial Fourier measurements $y = {\cal F}_{\Omega}f + \xi$, the estimation

{\footnotesize
\begin{equation}
\label{TV}
f^{\#} = \argmin_{g \in \C^{N \times N}} \| g \|_{TV}  \quad  \textrm{such that}  \quad \|\rho\circ( {\cal F}_{\Omega} g - y) \|_2   \leq \varepsilon\sqrt{m} ,
\end{equation}
}
approximates $f$ up to the noise level and best $s$-term approximation error of  its gradient:

{\footnotesize
\begin{equation}
\label{stable3}
\| f -f^{\#} \|_2 \lesssim  \frac{ \| \nabla f -  (\nabla f)_s \|_1}{\sqrt{s}}+ \varepsilon.
\end{equation}
}
\end{theorem}

Disregarding measurement noise, the error rate provided in Theorem~\ref{thm1} (and also the one in Theorem~\ref{thm2} below) is optimal up to logarithmic factors in the ambient image dimension. This follows from classical results about the Gel'fand width of the $\ell_1$-ball due to Kashin~\cite{Kas77:The-widths} and Garnaev--Gluskin~\cite{GG84:On-widths}.
As mentioned above, our noise model is non-standard, so the behavior for noisy signals is not covered by these lower bounds.

Our second result focuses on stable image reconstruction by $\ell_1$-minimization in the Haar wavelet transform domain. It is a direct consequence of applying the Fourier-wavelet incoherence estimates derived in Theorem \ref{thm:wBOS} to Theorem \ref{Bincoherent}.   

\begin{theorem}
\label{thm2}
Fix integers $N = 2^p, m,$ and $s$ such that $s\gtrsim \log(N)$ and 

{\footnotesize
\begin{equation}
\label{m:wav}
m \gtrsim s\log^{3}(s)\log^2(N).
\end{equation}
}
Select $m$ frequencies $\Omega = \{ (\omega_1^j, \omega_2^j ) \}_{j=1}^m \subset \{-N/2+1, \dots, N/2\}^2$ i.i.d. according to
the density $\eta$ as in \eqref{inverseD} and  assume again that the noise vector $\xi = (\xi_j)_{j=1}^m$ satisfies the weighted $\ell_2$-constraint with weight $\rho$ and noise level $\varepsilon$ as in Theorem~\ref{thm1}.  
Then with probability exceeding $1 - N^{-C\log^3(s)}$, the following holds for all images $f \in \C^{N \times N}$:  Given noisy measurements $y = {\cal F}_{\Omega}f + \xi$, the estimation

{\footnotesize
\begin{equation}
\label{l1-Haar}
f^{\#} = \argmin_{g \in \C^{N \times N}} \| {\cal H}g \|_1  \quad  \textrm{such that}  \quad \|\rho \circ ( {\cal F}_{\Omega} g - y)\|_2   \leq \varepsilon\sqrt{m}
\end{equation}
}
approximates $f$ up to the noise level and best $s$-term approximation error in the bivariate Haar basis: 

{\footnotesize
\begin{equation}
\label{stable3-Haar}
\| f -f^{\#} \|_2 \lesssim \frac{ \|  {\cal H}f -  ({\cal H}f)_s \|_1}{\sqrt{s}}+ \varepsilon.
\end{equation}
}
\end{theorem}
\vspace{3mm}
Even though the required number of samples $m$ in Theorem \ref{thm2} is smaller than the number of samples required for the total variation minimization guarantees in Theorem \ref{thm1}, one finds that total variation minimization requires fewer measurements empirically. This may be due to the fact that the gradient of a natural image has stronger sparsity than its Haar wavelet representation. For this reason we focus on total variation minimization. Independent of this observation, we strongly suspect that the additional logarithmic factors in the number of measurements  stated in Theorem \ref{thm1} are an artifact of the proof, and that it should be possible to strengthen the result to obtain a similar recovery guarantee with the number of measurements as in Theorem~\ref{thm2}. Moreover, one should be able to reduce the number of necessary log-factors with a RIP-less approach \cite{candesplan11}. These are important follow-up questions, as the current number of logarithmic factors may limit the direct 
applicability of our results to practical problems.

{\section{Compressive sensing background}\label{background}}
\subsection{The restricted isometry property}
Under certain assumptions on the matrix $\Phi: \C^N \rightarrow \C^m$ and the sparsity level $k$, any $k$-sparse  $x \in \C^N$ can be recovered from $y = \Phi x$ as the solution to the optimization problem:

{\footnotesize
\begin{equation*}
x = \argmin \| z \|_0  \quad  \textrm{such that}  \quad \Phi z = y
\end{equation*}
}
One of the fundamental results in compressive sensing is that this optimization problem, which is NP-hard in general, can be relaxed to an $\ell_1$-minimization problem if one asks that the matrix $\Phi$ restricted to any subset of $2k$ columns be well-conditioned.  This property is quantified via the so-called {\em restricted isometry property} as introduced in \cite{cata06}:
\begin{definition}[Restricted isometry property]
\label{def:RIP}
Let $\Phi \in \C^{m \times N}$. For $s \leq N$, the restricted isometry constant $\delta_s$ 
associated to $\Phi$ is the smallest number $\delta$ for which

{\footnotesize
\begin{equation}
(1-\delta) \|x \|_2^2 \leq \|\Phi x \|_2^2 \leq (1+\delta) \|x \|_2^2
\end{equation}
}
for all $s$-sparse vectors $x \in \C^N$. If $\delta_s \leq\delta$, one says that $\Phi$ has the  restricted isometry property (RIP) of order $s$ and level $\delta$.
\end{definition}
The restricted isometry property ensures stability: not only sparse vectors, but also compressible vectors can be recovered from the measurements via $\ell_1$-minimization.  It also ensures robustness to measurement errors.
% As mentioned above compressible vector is one which is well-approximated by sparse vectors,  in the sense that its best $s$-term approximation error $\sigma_s(f)_1 := \min\limits_{\tilde f \text{ $s$-sparse}}\| \tilde f-f\|_{\ell_1}$ is small.
\begin{proposition}[Sparse recovery for RIP matrices]
\label{Prop:RIP} 
Assume
that the restricted isometry constant $\delta_{5s}$ of $\Phi \in \C^{m \times N}$
satisfies $\delta_{5s} < \frac{1}{3}$.  Let $x \in \C^N$ and assume noisy measurements $y = \Phi x + \xi$ with $\| \xi \|_2 \leq \varepsilon$. Then

{\footnotesize
\begin{align}
\label{L_1}
x^{\#} = \arg \min_{z \in \C^N} \quad \| z \|_1 \mbox{ subject to } \quad \| \Phi z - y \|_2 \leq \varepsilon
\end{align}
}
satisfies

{\footnotesize
\begin{align}
\label{l2noise}
\|x - x^\#\|_2 \leq \frac{2\sigma_s(x)_1}{\sqrt{s}} + \varepsilon.
\end{align}
}
In particular, reconstruction is exact, $x^\# = x$, if $x$ is $s$-sparse and $\varepsilon = 0$.
\end{proposition}
There are stronger versions of this result which allow for weaker constraints on the restricted isometry constant \cite{moli}. However, our version is a corollary of the following proposition, which appears as Proposition 2 in \cite{nw1}, and generalizes the results from \cite{crt}. This proposition will also play an important role in the proof of our main results.

\begin{proposition}[Stable recovery for RIP matrices, \cite{nw1}]
\label{cone-tube}
Suppose that $\gamma\geq 1$ and $\Phi \in \C^{m \times N}$ satisfies the restricted isometry property of order at least $5k\gamma^2$ and level $\delta < 1/3$, and suppose that $ u \in \C^N$ satisfies a tube constraint

{\footnotesize
\begin{equation*}
\| \Phi  u \|_2 \lesssim \varepsilon.
\end{equation*}
}
Suppose further that for a subset $S$ of cardinality $|S| = k$, the signal $ u$ satisfies a cone constraint

{\footnotesize
\begin{equation}
\label{cc}
\|  u_{S^c} \|_1 \leq \gamma \|  u_S \|_1 + {\xi}.
\end{equation}}
Then 

{\footnotesize
\begin{equation}\label{eq:h2}
\|  u \|_2 \lesssim \frac{{\xi}}{\gamma\sqrt{k}} + \varepsilon.
\end{equation}
}
\end{proposition}
\noindent Indeed, Proposition~\ref{Prop:RIP} follows from Proposition~\ref{cone-tube} by noting that the minimality of $x^\#$ implies a cone constraint for the residual $x - x^\#$ over the support of the $s$ largest-magnitude entries of $x$.  The proof of Proposition \ref{cone-tube} can be found in \cite{nw1}.

\subsection{Bounded orthonormal systems}
While the strongest known results on the restricted isometry property concern random matrices with independent entries such as Gaussian or Bernoulli, a scenario that has proven particularly useful for applications is that of structured random matrices with rows chosen from a basis incoherent to the basis inducing sparsity (see below for a detailed discussion on the concept of incoherence).  The resulting sampling schemes correspond to \emph{bounded orthonormal systems}, and such systems have been extensively studied in the compressive sensing literature (see \cite{ra09-1} for an expository article including many references). 

\begin{definition}[Bounded orthonormal system] 

Consider a set $T$ equipped with probability measure $\nu$.
\begin{itemize}
 \item   A set of functions $\{\psi_j:T\rightarrow \C, \hspace{.5mm} j\in[N]\}$ is called an {\em orthonormal system with respect to $\nu$} if $\int_T \bar{\psi}_j(x)\psi_k(x)d\nu(x)=\delta_{jk}$, where $\delta_{jk}$ denotes the Kronecker delta.
 \item An orthonormal system is said to be \emph{bounded} with bound $K$ if $\sup_{j\in[N]} \|\psi_j(x)\|_\infty\leq K$.
\end{itemize}\label{def:BOS}
\end{definition}
For example, the basis of complex exponentials $\psi_j(x) = \exp{(i 2 \pi j x)}$  forms a bounded orthonormal system with optimally small constant $K=1$ with respect to the uniform measure on $T = \{0,\frac{1}{N}, \dots, \frac{N-1}{N} \}$, and $d$-dimensional tensor products of complex exponentials form bounded orthonormal systems with respect to the uniform measure on the set $T^d$. A {\em random sample} of an orthonormal system is the vector $(\psi_1(x), \hdots, \psi_N(x))$, where $x$ is a random variable drawn according to the associated distribution $\nu$.   Any matrix whose rows are independent random samples of a bounded orthonormal system, such as the uniformly subsampled discrete Fourier matrix, will have the restricted isometry property:

\begin{proposition}[RIP for bounded orthonormal systems, \cite{ra09-1}]
\label{BOS} 
Consider 
the matrix $\Psi \in \C^{m \times N}$ whose rows are independent random samples of an 
orthonormal system $\{ \psi_j$, $j \in [N] \}$ with bound $K\geq 1$ and orthogonalization measure $\nu$.
If 

{\footnotesize
\begin{equation}\label{BOS:RIP:cond}
m \gtrsim \delta^{-2} K^2 s \log^3(s) \log(N),
\end{equation}
}
for some $s\gtrsim \log(N)$\footnote{For matrices consisting of uniformly subsampled rows of the discrete Fourier matrix, it has been shown in \cite{CGV12} that this constraint is not necessary.}, then with probability at least 
$1-N^{-C \log^3(s)},$ %$1-\exp\big(- \gamma \log^2(s) \log(m) \log(N)\big) \geq 1-N^{-\gamma}$
the restricted isometry constant $\delta_s$ of $\frac{1}{\sqrt{m}} \Psi$ satisfies $\delta_s \leq \delta$.
\end{proposition}

 An important  special case of a bounded orthonormal system arises by sampling a function with sparse representation in one basis using measurements from a different, incoherent basis.  
 The {\em mutual coherence} between a unitary matrix $A \in \C^{N \times N}$ with rows $(a_j)_{j=1}^N$ and a unitary matrix $B \in \C^{N \times N}$ with rows $(b_j)_{j=1}^N$ is given by
 
 {\footnotesize
\begin{equation*}
 \mu(A,B)= \sup_{j,k} |\scalprod{a_j,b_k}|
\end{equation*} 
}
 The smallest possible mutual coherence is $\mu = N^{-1/2}$, as realized by the discrete Fourier matrix and the identity matrix.  We call two orthonormal bases $A$ and $B$ \emph{mutually incoherent} if $\mu=O(N^{-1/2})$ or $\mu = O(\log^\alpha(N)N^{-1/2})$.  In this case, the rows $(\tilde b_j)_{j=1}^N$ of the basis $\widetilde B = \sqrt{N} BA^*$ constitute a bounded orthonormal system with respect to the uniform measure. Propositions~\ref{BOS} and \ref{Prop:RIP} then imply that, with high probability, signals $f \in \C^N$ of the form $f = Ax$ for $x$ sparse can be stably reconstructed from uniformly subsampled measurements $y = B f = \widetilde B x$, as $\widetilde B$ has the restricted isometry property.
 
\begin{corollary}[RIP for incoherent systems, \cite{RV08:sparse}]
\label{thm:incoh} 
Consider orthonormal bases $A, B \in \C^{N \times N}$ with mutual coherence bounded by
$\mu(A,B) \leq K N^{-1/2}.$
Fix $\delta > 0$ and integers $N, m$, and $s$ such that $s\gtrsim \log(N)$ and

{\footnotesize
\begin{equation}\label{incoh:RIP:cond}
m \gtrsim \delta^{-2} K^2 s \log^3(s) \log(N).
\end{equation}
}
Consider the matrix $\Phi \in \C^{m \times N}$ formed by uniformly subsampling $m$ rows i.i.d. from the the matrix $\widetilde B = \sqrt{N} B A^*$. 
Then with probability at least 
$1-N^{-c \log^3(s)},$ %$1-\exp\big(- \gamma \log^2(s) \log(m) \log(N)\big) \geq 1-N^{-\gamma}$
the restricted isometry constant $\delta_s$ of $\frac{1}{\sqrt{m}} \Phi$ satisfies $\delta_s \leq \delta$. 
\end{corollary}

{\section{Local coherence}\label{riplocal}}

The sparse recovery results in Corollary \ref{thm:incoh} based on mutual coherence do not take advantage of the full range of applicability of bounded orthonormal systems.  As argued in \cite{rw11},  Proposition~\ref{BOS} implies comparable sparse recovery guarantees for a much wider class of sampling/sparsity bases through preconditioning resampled systems.  In the following, we formalize this approach through the notion of local coherence.  

\begin{definition}[Local coherence]
The \emph{local coherence} of an orthonormal basis $\{\varphi_j\}_{j=1 }^N$ of $\C^N$ with respect to the orthonormal  basis $\{\psi_k\}_{k=1 }^N$ of $\C^N$ is the function $\mu^{loc}(\Phi, \Psi) \in \R^N$ defined coordinate-wise by 

{\footnotesize
\begin{equation*} \mu^{loc}_j(\Phi, \Psi) =  \sup\limits_{1\leq k\leq N} |\langle \varphi_j, \psi_k\rangle|, \quad \quad  j = 1,2,\dots, N
\end{equation*}
}
\end{definition}

The following result shows that we can reduce  the number of measurements $m$ in \eqref{thm:incoh} by replacing the bound $K$ on the coherence in \eqref{incoh:RIP:cond} by a bound on the $\ell_2$-norm of the local coherence, provided we sample rows from $\Phi$ appropriately using the  local coherence function.  It can be seen as a direct finite-dimensional analog to Theorem~2.1 in \cite{rw11}, but for completeness, we include a short self-contained proof. 
 
\begin{theorem}\label{thm:wBOS}
 Let $\Phi=\{\varphi_j\}_{j=1}^N$ and  $\Psi =\{\psi_k\}_{k=1}^N$ be orthonormal bases of $\C^N$. Assume the local coherence of $\Phi$ with respect to $\Psi$  is pointwise bounded by the function $\kappa$, that is  $ \sup\limits_{1\leq k\leq N} |\langle \varphi_j, \psi_k\rangle| \leq \kappa_j$. 
Let $s\gtrsim \log(N)$, suppose

{\footnotesize
 \begin{equation}\label{wBOS:RIP:cond}
m \gtrsim \delta^{-2} \|\kappa \|_2^2 s \log^3(s) \log(N),
\end{equation}
}
and choose $m$ (possibly not distinct) indices $j \in \Omega \subset [N]$ i.i.d. from the probability measure $\nu$ on $[N]$ given by
 
 {\footnotesize
\begin{equation*}
 \nu(j) = \frac{\kappa^2_j}{\|\kappa \|_2^2 }.
\end{equation*}
}

 Consider the matrix $A \in \C^{m \times N}$ with entries

 {\footnotesize
\begin{equation}\label{def:wPhi:matrix}
A_{j,k} = \langle \varphi_j, \psi_k\rangle, \quad j \in \Omega, k \in [N],
\end{equation}
}
and consider the diagonal matrix $D = \operatorname{diag}(d) \in \C^{N}$ with $d_{j} = \| \kappa \|_2 / \kappa_j$.
Then with probability at least 
$1-N^{-c \log^3(s)},$ 
the restricted isometry constant $\delta_s$ of the preconditioned matrix $\frac{1}{\sqrt{m}} D A$ satisfies $\delta_s \leq \delta$.
\end{theorem}
\begin{proof}
Note that as the matrix $\Psi$ with rows $\psi_k$ is unitary, the vectors $\eta_j:=\Psi \phi_j$, $j=1, \dots, N$, form an orthonormal system with respect to the uniform measure on $[N]$ as well. 
We show that the system $\{ \widetilde \eta_j \} = \{ d_{j} \eta_j \}$ is an orthonormal system with respect to $\nu$ in the sense of Definition~\ref{def:BOS}. Indeed,
{\footnotesize
\begin{align}
 \sum\limits_{j=1}^N \widetilde\eta_j(k_1) \widetilde\eta_j(k_2) \nu(j)  &= \sum\limits_{j=1}^N \Big( \frac{\|\kappa \|_2}{\kappa_j}\eta_j(k_1) \Big) \Big( \frac{\|\kappa \|_2}{\kappa_j}\eta_j(k_2) \Big) \frac{\kappa^2_j}{\|\kappa\|_2^2} \\&= \sum\limits_{j=1}^N \eta_j(k_1) \eta_j(k_2)  =\delta_{k_1,k_2};
\end{align}
}
hence the $\widetilde\eta_j$ form an orthonormal system with respect to $\nu$. Noting that $|\eta_j(k)|=|\langle\varphi_j, \psi_k\rangle|\leq \kappa_j$ and hence this system is bounded with bound $\|\kappa \|_2$, the result follows from Proposition~\ref{BOS}.
\end{proof}
\begin{remark}\label{rem:UB}
Note that the local coherence not only appears in the embedding dimension $m$, but also in the sampling measure. Hence a priori, one cannot guarantee the optimal embedding dimension if one only has suboptimal bounds for the local coherence. That is why the sampling measure in Theorem~\ref{thm:wBOS} is defined via the (known) upper bounds $\kappa$ and $\|\kappa\|_2$ rather than the (usually unknown) exact values $\mu_{loc}$ and $\|\mu_{loc}\|_2$, showing that suboptimal bounds still lead to meaningful bounds on the embedding dimension.
\end{remark}
\begin{remark}
 For $\mu\leq K N^{-1/2}$ (as in Corollary~\ref{thm:incoh}), one has $\| \mu^{loc} \|_2  \leq K$ , so Theorem~\ref{thm:wBOS} is a direct generalization of Corollary~\ref{thm:incoh}. As one has equality if and only if $\mu^{loc}$ is constant, however, Theorem~\ref{thm:wBOS} will be stronger in most cases.
\end{remark}

{\section{Local coherence estimates for frequencies and wavelets}\label{local}}
%\subsection{Incoherence estimates}
Due to the tensor product structure of both of these bases,  the two-dimensional local coherence of the two-dimensional Fourier basis with respect to bivariate Haar wavelets will follow by first bounding the local coherence of the one-dimensional Fourier basis with respect to the set of univariate building block functions of the bivariate Haar basis.

\begin{lemma}\label{lem:univariate}
Fix $N = 2^p$ with $p\in\N$.  For the space $\C^{N}$, the one-dimensional Fourier basis vectors $\varphi_k$, $k \neq 0$, and the one-dimensional Haar wavelet basis building blocks $h^e_{n,k}$, $e=0,1$, satisfy the incoherence relation

{\footnotesize
\begin{equation}
 | \langle \varphi_k, h^e_{n,\ell}\rangle | \leq  \min\Big( \frac{6\cdot 2^{\frac{n}{2}}}{|k|}, 3\pi 2^{-\frac{n}{2}} \Big).
\end{equation}
}
\end{lemma}
\begin{proof}
 We estimate
 
{\footnotesize
\begin{align}
 \langle \varphi_k, h^e_{n,\ell}\rangle =& \sum_{j=2^{p-n} \ell}^{2^{p-n}\ell+2^{p-n-1}-1}2^{\frac{n-p}{2}}2^{-\frac{p}{2}} e^{2\pi i 2^{-p}k j}\\ &+(-1)^e\sum_{j=2^{p-n} \ell+2^{p-n-1}}^{2^{p-n}\ell+2^{p-n}-1} 2^{\frac{n-p}{2}} 2^{-\frac{p}{2}} e^{2\pi i 2^{-p}k j}\\
=& e^{2\pi i 2^{-n} \ell k}\left(1+(-1)^e e^{2\pi i 2^{-n-1}k} \right)\nonumber\\ &\cdot 2^{\frac{n}{2}-p}\sum_{j=0}^{2^{p-n-1}-1} e^{2\pi i 2^{-p} k j} \label{eqn:summation}\\
=& e^{2\pi i 2^{-n} \ell k}\left(1+(-1)^e e^{2\pi i 2^{-n-1}k} \right)\nonumber\\ &\cdot  2^{\frac{n}{2}-p}\frac{1-e^{2\pi i 2^{-n-1}k}}{1 -e^{2\pi i 2^{-p}k}}. \label{eqn:geomser}
\end{align}
}

To estimate this expression, we note that
{\footnotesize
\begin{equation}
 |1-e^{2\pi i 2^{-n-1}k}|\leq \min(2, \pi  2^{-n}|k|)%\leq 2^{1-\alpha} \pi^{\alpha} 2^{-n\alpha}|\ell|^{\alpha}
 \label{eq:sincbound}
 \end{equation}
 }
 and distinguish two cases:

If $0\neq |k|\leq 2^{p-2}$, we bound $|1 - e^{2\pi i 2^{-p} k}|\geq 2^{-p} |k|$ and apply \eqref{eq:sincbound} to obtain

{\footnotesize
\begin{align}
 |\langle \varphi_k, h_{n,\ell}^e\rangle |&\leq 2\cdot 2^{\frac{n}{2}-p} \frac{ \min(2, \pi  2^{-n}|k|)}{2^{-p} |k|}\\
&\leq  \min(\frac{4\cdot 2^{\frac{n}{2}}}{|k|}, 2\pi 2^{-\frac{n}{2}}).
\end{align}
}
For $2^{p-2}<|k|\leq 2^{p-1}$, and hence $2^{-p} \leq \frac{1}{2}|k|^{-1}$, we note that $|1 - e^{2\pi i 2^{-p}k}|\geq \frac{\sqrt{2}}{2}$
and bound, again using \eqref{eq:sincbound},

{\footnotesize
\begin{align}
 |\langle \varphi_k, h^e_{n,\ell}\rangle |&\leq  2\cdot 2^{\frac{n}{2}} |k|^{-1} \frac{ \min(2, \pi  2^{-n}|k|)}{\frac{\sqrt{2}}{2}}\\
& \leq  \min \Big(\frac{6\cdot 2^{\frac{n}{2}}}{|k|}, 3\pi 2^{-\frac{n}{2}} \Big).
\end{align}
}
\end{proof}

This lemma enables us to derive the following incoherence estimates for the bivariate case.
\begin{theorem}
\label{Bincoherent}
Let $N=2^p$ for $\N\ni p\geq 8$. Then the local coherence $\mu^{loc}$ of the orthonormal two-dimensional Fourier basis $\{\varphi_{k_1,k_2}\}$ with respect to the orthonormal bivariate Haar wavelet basis $\{h^e_{n,\ell}\} $ in $\C^{N\times N}$, as defined in \eqref{defFour} and \eqref{bvhaar}, respectively, is bounded by

{\footnotesize
\begin{align}
 \mu^{loc}_{k_1, k_2}&\leq \kappa(k_1, k_2):= \min\left(
                    1,\frac{18\pi}{\max(|k_1|,|k_2|)} \right)\\
                   &\leq \kappa'(k_1,k_2):= \min\left(1,
                    \frac{18\pi\sqrt{2}}{\left(|k_1|^2+|k_2|^2\right)^{1/2}}\right),
\end{align}
}
and one has $\|\kappa \|_2\leq \|\kappa'\|_2\leq 52\sqrt{p} = 52\sqrt{\log_2(N)}$.
\end{theorem}

\begin{proof}
First note that the bivariate Fourier coefficients decompose into the product of univariate Fourier coefficients:
 
 {\footnotesize
 \begin{align}
 \langle \varphi_{k_1, k_2}, h_{n,\ell}^e \rangle&= \langle \varphi_{k_1}, h^{e_1}_{n,\ell_1} \rangle\langle \varphi_{k_2}, h^{e_2}_{n,\ell_2}\rangle.
 %
%  \sum_{u = 1}^{2^p}  e^{-2\pi i \ell_1 u } 2^{n/2-p/2} H^{e_1}(2^n u - k_1) \sum_{v = 1}^{2^p} e^{-2\pi i \ell_2 v } 2^{n/2-p/2} H^{e_2} (2^n v - k_2) \nonumber
 \end{align}
 }
For $k_i\neq 0$, the factors can be bounded using Lemma~\ref{lem:univariate}, which, for $k_1\neq 0 \neq k_2$, 
yields the bound

{\footnotesize
\begin{align*}
 |\langle \varphi_{k_1, k_2}, h_{n,\ell}^e \rangle| &\leq \min\Big(\frac{6\cdot 2^{\frac{n}{2}}}{|k_1|}, 3\pi 2^{-\frac{n}{2}} \Big)\min\Big(\frac{6\cdot 2^{\frac{n}{2}}}{|k_2|}, 3\pi 2^{-\frac{n}{2}} \Big)\\ &\leq \frac{18\pi}{\max(|k_1|,|k_2|)}.
\end{align*}
}

Next we consider the case where either $k_1=0$ or $k_2=0$; w.l.o.g., assume $k_1=0$. We use that in one dimension, we have   $\langle\varphi_0, h^1_{n,\ell}\rangle=0$ as well as $\langle\varphi_0, h^0_{n,\ell}\rangle=2^{-\frac{n}{2}}$. So we only need to consider the case that $e_1=0$ and hence $e_2=1$. Thus we obtain

{\footnotesize
\begin{equation*}
 |\langle \varphi_{0, k_2}, h_{n,\ell}^e \rangle| \leq 2^{-\frac{n}{2}} \frac{6\cdot 2^{\frac{n}{2}}}{| k_2|}= \frac{6}{\max(| k_1|,|k_2|)}.
\end{equation*}
}
In both cases, we obtain $ \mu^{loc}_{k_1, k_2}\leq \frac{18\pi}{\max(|k_1|,|k_2|)}$. The bound  $ \mu^{loc}_{k_1, k_2}\leq 1$ follows directly from the Cauchy-Schwarz inequality. We conclude  $ \mu^{loc}_{k_1, k_2}\leq \kappa(k_1, k_2)\leq \kappa'(k_1, k_2)$.

For the $\ell_2$-bound, we use an integral estimate,

{\footnotesize
\begin{align}
\|\kappa '\|_2^2
  \leq & \#\{(k_1,k_2): k_1^2+ k_2^2\leq 648 \pi^2\} \nonumber\\ &+ \sum_{\substack{{k_1,k_2=-2^{p-1}+1}
\\ |k_1|^2+|k_2|^2> 648\pi^2}}^{2^{p-1} } \frac{648\pi^2}{|k_1|^2+ |k_2|^2}\\
 \leq & 20600+ \iint\limits_{r=18\pi\sqrt{2}-1}^{2^{p-\frac{1}{2}}} 18\pi \sqrt{2} r^{-1} dr d\phi \\
 \leq & 17200+502\log_2(N) \leq 2700 \log_2(N) =2700 p,
 \end{align}
}
where we used that $p\geq 8$. Taking square root implies the result.
\end{proof}
\begin{remark}
 We believe that the factor of $\sqrt{\log_2 N}$ which appears in the proposition is due to lack of smoothness for the Haar wavelets. Hence we hope this factor can be removed by considering smoother wavelets.
\end{remark}

As the infimum of a strictly decreasing function and a strictly increasing function is bounded uniformly by its value at the intersection point of the two functions, Lemma \ref{lem:univariate} also gives frequency-dependent bounds for the  local coherence between frequencies and wavelets in the univariate setting.
\begin{corollary}
Fix $N = 2^p$ with $p\in\N$.  For the space $\C^{N}$, the one-dimensional Fourier basis vectors $\varphi_k$, $k \neq 0$, and the one-dimensional Haar wavelets  satisfy the incoherence relation

{\footnotesize
\begin{equation}
 | \langle \varphi_k, h_{n,\ell}\rangle | \leq  3\sqrt{2\pi}/ \sqrt{k}.
\end{equation}
}
\end{corollary}

{\section{Recovery guarantees}\label{proofs}}
In this section we present proofs of the main results.
\subsection{Proof of Theorem \ref{thm2}}
The proof of Theorem~\ref{thm2} concerning recovery from $\ell_1$-minimization in the bivariate Haar transform domain follows by combining the local incoherence estimate of Theorem~\ref{Bincoherent} with the local coherence based reconstruction guarantees of Theorem~\ref{wBOS:RIP:cond}.  Under the conditions of Theorem~\ref{wBOS:RIP:cond}, the stated recovery results follow directly from Theorem~\ref{Prop:RIP}. The weighted $\ell_2$-norm in the noise model is a consequence of the preconditioning. 

\subsection{Preliminary lemmas for the proof of Theorem \ref{thm1}}

The proof of Theorem~\ref{thm1} proceeds along similar lines to that of Theorem~\ref{thm2}, but we need a few more preliminary results relating the bivariate Haar transform to the gradient transform.  The first result, Proposition \ref{cdpx}, is derived from a more general statement involving the continuous bivariate Haar system and the bounded variation seminorm from \cite{cdpx}.
 \begin{proposition}
\label{cdpx}
Suppose $f \in \C^{N^2}$ has mean zero, and suppose its bivariate Haar transform $w = {\cal H} f \in \C^{N^2}$ is arranged such that $w_{(k)}$ is the $k$-th largest-magnitude coefficient.  Then there is a universal constant $C > 0$ such that for all $k \geq 1$, 

{\footnotesize
\begin{equation*}
| w_{(k)} | \leq C \frac{ \| f \|_{TV}}{k}
\end{equation*}
}
\end{proposition}
\noindent See \cite{nw1} for a derivation of Proposition \ref{cdpx} from Theorem $8.1$ of \cite{cdpx}.  

We will also need two lemmas about the bivariate Haar system.
\begin{lemma}
\label{edge}
Let $N = 2^p$.  For any indices $(t_1, t_2)$ and $(t_1, t_2 +1),$ there are at most $6p$ bivariate Haar wavelets $h_{n,\ell}^e$ satisfying $| h_{n,\ell}^e(t_1, t_2+1) - h_{n,\ell}^e(t_1, t_2)| > 0$.
\end{lemma}
\begin{proof}
The lemma follows by showing that for fixed dyadic scale $n$ in $0 < n \leq p$, there are at most 6 Haar wavelets with edge length $2^{p-n}$ satisfying $| h_{n,\ell}^e(t_1, t_2+1) - h_{n,\ell}^e(t_1, t_2)| > 0$. If the edge between $(t_1, t_2)$ and $(t_1, t_2+1)$ coincides with a dyadic edge at scale $n$, then the 3 wavelets supported on each of the bordering dyadic squares transition from being zero to nonzero along this edge.  On the other hand, if $(t_1, t_2)$ coincides with a dyadic edge at dyadic scale $n+1$ but does not coincide with a dyadic edge at scale $n$, then 2 of the 3 wavelets supported on the dyadic square having $(t_1, t_2+1), (t_1, t_2)$ as a center edge can satisfy the stated bound.      
\end{proof}

\begin{lemma}
\label{l1haar}

{\footnotesize
\begin{equation*} 
\| \nabla h_{n,\ell}^e \|_1 \leq 8 \quad \quad \forall n,\ell,e.
\end{equation*}
}
\end{lemma}
\begin{proof}
$h_{n,\ell}^e$ is supported on a dyadic square of side-length $2^{p-n}$, and on its support, its absolute value is constant, $| h_{n,\ell}^e | = 2^{n-p}$. Thus at the four boundary edges of the square, there is a jump of $2^{n-p}$, at the (at most two) dyadic edges in the middle of the square where the sign changes there is a jump of $2\cdot 2^{n-p}$. Hence $\| \nabla h_{n,\ell}^e \|_1 \leq \| \nabla h_{n,\ell}^{\{1,1\}} \|_1 \leq 8 \cdot 2^{p-n} \cdot 2^{n-p} = 8$.
\end{proof}

\vspace{3mm}

\noindent We are now ready to prove Theorem \ref{thm1}.  

\subsection{Proof of Theorem \ref{thm1}}

Recall that ${\cal H}: \C^{N^2} \rightarrow \C^{N^2}$ denotes the bivariate Haar transformation $f \mapsto \big( \scalprod{f,h^e_{n,\ell}} \big)_{n,\ell,e}.$ Let $w^f_{(1)}$ be the Haar coefficient corresponding to the constant wavelet, and let $w^f_{(j)}$, for $j \geq 2$, denote the $j-1$-st largest-magnitude Haar coefficient among the remaining, and let $h_{(j)}$ denote the associated Haar wavelet.  We use this slightly modified ordering because Proposition \ref{cdpx} applies only to mean-zero images.

Let $D\in \C^{N^2\times N^2}$ be the diagonal matrix encoding the weights in the noise model, i.e., $D = \operatorname{diag}(\rho)$, where, for $\kappa'$ as in Theorem~\ref{Bincoherent}, $\rho({k_1, k_2}) = \| \kappa' \|_2 / \kappa'(k_1, k_2) = C \sqrt{\log_2(N)}  \max \Big(1, (|k_1|^2 + |k_2|^2 )^{1/2}/18\pi \Big)$. Note that $D g \equiv \rho\circ g$.

 By Theorem~\ref{thm:wBOS} combined with the bivariate incoherence estimates from Theorem \ref{Bincoherent}, we know that with high probability ${\cal A} := \frac{1}{\sqrt{m}} D{\cal F}_{\Omega} {\cal H}^{*}$ has the restricted isometry property of order $s$ and level $\delta$ once

 {\footnotesize
 \begin{equation*}
  m \gtrsim  s\delta^{-2}\log^3(s) \log^2{N}.
\end{equation*}
}
\noindent Thus, for the stated number of measurements $m$ with an appropriate hidden constant, we can assume that ${\cal A}$ has the restricted isometry property of order 

{\footnotesize
\begin{equation*}
\overline{s} = 24 {\widetilde C}^2 s \log^3(N), 
\end{equation*}
}
where the exact value of the constant $\widetilde C$ will be determined below.
In the remainder of the proof we show that this event implies the result. 

\vspace{3mm}

\noindent Let $ u = f - f^{\#}$ denote the residual error of \eqref{TV}.  Then we have
\begin{itemize}
\item {\bf Cone Constraint on $\nabla  u$.} \quad  Let $S$ denote the support of the best $s$-sparse approximation to $\nabla f$.  Since $f^{\#} = f - u$ is the minimizer of (TV) and $f$ is also a feasible solution,

{\footnotesize
\begin{align*}
\| (\nabla f)_S \|_1 &- \| (\nabla  u)_S\|_1 - \| (\nabla f)_{S^c}\|_1 + \| (\nabla  u)_{S^c}\|_1 \\
&\leq \| (\nabla f)_S - (\nabla  u)_S\|_1 + \| (\nabla f)_{S^c} - (\nabla  u)_{S^c}\|_1\\
&= \| \nabla f^{\#} \|_1\\
&\leq \| \nabla f \|_1\\
&= \| (\nabla f)_S\|_1 + \| (\nabla f)_{S^c}\|_1
\end{align*}
}
Rearranging yields the cone constraint

{\footnotesize
\begin{equation}
\label{coned}
\| (\nabla u)_{S^c} \|_1 \leq \| (\nabla u)_S \|_1 + 2\| \nabla f - (\nabla f)_S \|_1.
\end{equation}
}
\item {\bf Cone Constraint on $w^u = {\cal H} u$.} \quad
\noindent Proposition \ref{cdpx} allows us to pass from a cone constraint on the gradient to a cone constraint on the Haar transform.  More specifically, we obtain

{\footnotesize
\begin{equation*}
 | w^u_{(j+1)} | \leq C \frac{ \| \nabla  u \|_1}{j},
\end{equation*}
}
recalling that $w^u_{(1)}$ is the coefficient associated to the constant wavelet.
Now consider the set $\widetilde S$ consisting of the $s$ edges indexed by $S$.  By Lemma \ref{edge},  the set $\Lambda$  indexing those wavelets which change sign across edges in $\widetilde S$ has cardinality at most $| \Lambda | = 6s \log(N)$.  Decompose $ u$ as

{\footnotesize
\begin{equation}
\label{d_decomp}
 u = \sum_{j} w^u_{(j)} h_{(j)} =  \sum_{j \in \Lambda} w^u_{(j)} h_{(j)} + \sum_{j \in \Lambda^c} w^u_{(j)} h_{(j)} =:  u_{\Lambda} +  u_{\Lambda^c}
\end{equation}
}
and note that by linearity of the gradient, 

{\footnotesize
\begin{equation*}
 \nabla  u = \nabla  u_{\Lambda} + \nabla  u_{\Lambda^c}.   
\end{equation*}
}
Now, by construction of the set $\Lambda$, we have that $(  \nabla  u_{\Lambda^c} )_S = 0$ and so $  (\nabla  u)_S  = (\nabla  u_{\Lambda} )_S$.
By Lemma \ref{l1haar} and the triangle inequality, 

{\footnotesize
\begin{align*}
\|( \nabla  u )_S \|_1 &= \| (\nabla  u_{\Lambda} )_S \|_1 \leq \|  \nabla { u_{\Lambda}} \|_1  \\
&\leq \sum_{j \in \Lambda} | w_{(j)} |  \|\nabla h_{(j)} \|_1 \\
&\leq 8 \sum_{j \in \Lambda} | w_{(j)} |.
\end{align*}
}
\noindent Combined with Proposition \ref{cdpx} concerning the decay of the wavelet coefficients and the cone constraint \eqref{coned}, and letting 

{\footnotesize
\begin{equation*}
\widetilde{s} = 6s \log(N) = | \Lambda |,
\end{equation*}
}
\noindent this gives rise to a cone constraint on the wavelet coefficients:

 {\footnotesize
\begin{align*}
\sum_{j = \widetilde{s} + 1}^{N^2} &| w^u_{(j)} | \leq \sum_{j = s+1}^{N^2} | w^u_{(j)} | \nonumber \\
 &\leq C \log(N^2/s)  \| \nabla  u \|_{1}\nonumber \\
&= C \log(N^2/s) \Big( \| (\nabla  u)_S \|_1 +  \| (\nabla  u)_{S^c} \|_1  \Big)\nonumber \\ \nonumber \\
&\leq C \log(N^2/s)\Big( \|2 (\nabla  u)_S \|_1 + 2\| \nabla f - (\nabla f)_S \|_1\Big) \nonumber \\ \nonumber \\
&\leq C \log(N^2/s)\Big(  16 \sum_{j \in \Lambda} | w_{(j)} | + 2\| \nabla f - (\nabla f)_S \|_1\Big)   \nonumber \\ \nonumber \\
&\leq \widetilde C  \log(N^2/s)\Big( \sum_{j =1}^{\widetilde{s} } | w_{(j)} | + \| \nabla f - (\nabla f)_S \|_1\Big)   
\end{align*}
}

\item {\bf Tube constraint,  $\| {\cal A} {\cal H} u \|_2 \leq \sqrt{2} \varepsilon.$} \\
 By assumption, ${\cal A} =\frac{1}{\sqrt{m}}D{\cal F}_{\Omega} {\cal H}^{*}: \C^{N^2} \rightarrow \C^m$ has the RIP of order $\overline s>s$. 
Also by assumption, $\| D {\cal F}_{\Omega} f - D y \|_{2}= \| \rho \circ ({\cal F}_{\Omega} f -  y)\|_{2}\leq  \sqrt{m} \varepsilon$, so $f$ is a feasible solution to \eqref{TV}.

\vspace{1mm}

 Since both $f$ and $f^{\#}$ are in the feasible region of \eqref{TV}, we have for $ u = f - f^{\#}$,
{\footnotesize
\begin{align*}
m \| {\cal A} {\cal H}  u \|^2_2 &= \| D{\cal F}_{\Omega}{{\cal H}}^{*} {\cal H}  u \|^2_2 = \| D{\cal F}_{\Omega} u \|_2^2 \nonumber \\%\nonumber \\
&\leq \| D{\cal F}_{\Omega}f - Dy \|_2^2 + \| D{\cal F}_{\Omega}f^{\#} - D y \|_2^2 \nonumber \\
&\leq 2m \varepsilon^2.
\end{align*}
}

\item Using the derived cone and tube constraints on ${\cal H} u$ along with the assumed RIP bound on ${\cal A}$, the proof is complete by applying Proposition~\ref{cone-tube} using $\gamma = \widetilde C\log(N^2/s)\leq 2 \widetilde C\log(N)$, $k = 6s\log N$, and $\xi = \widetilde C\log(N^2/s) \| \nabla f - (\nabla f )_S \|_1$. In fact, this  is where we need that the RIP order is $\overline s$, to accommodate for the factors $\gamma$ and $k$. \qed

\end{itemize}

\section{Numerical illustrations}\label{sec:numerics}
In this section, we will provide numerical examples for our results. As there have been papers entirely devoted to the empirical investigation of optimal sampling strategies \cite{WA10}, the goal will be to illustrate our results rather than provide a thorough empirical analysis.

First, we consider a $256\times 256$ spine image and visually compare the reconstruction quality for different spine images. While the inferior reconstruction quality for uniform sampling is obvious, the difference between variable density sampling and using only the low frequencies is less apparent, both visually and in the reconstruction error. 

%However, zooming in, one can see that fine details are better captured in the variable density sampling schemes compared to the low-frequency-only schemes.

\begin{figure}[h!]
\begin{center}
\mbox{
\subfigure[Original image]{
 \includegraphics[height=1.8cm,width=1.8cm]{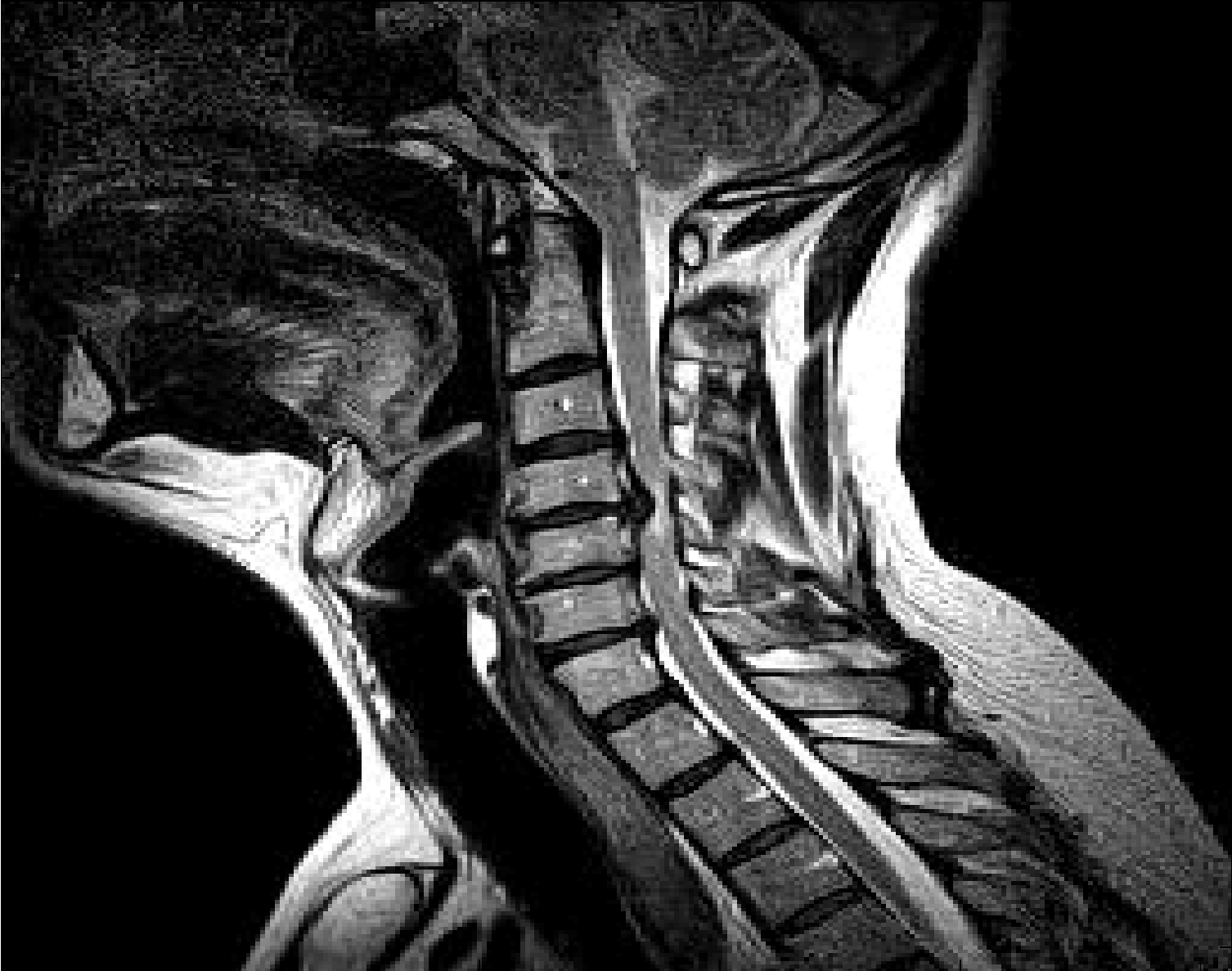}
 \includegraphics[height=1.8cm, width=1.8cm]{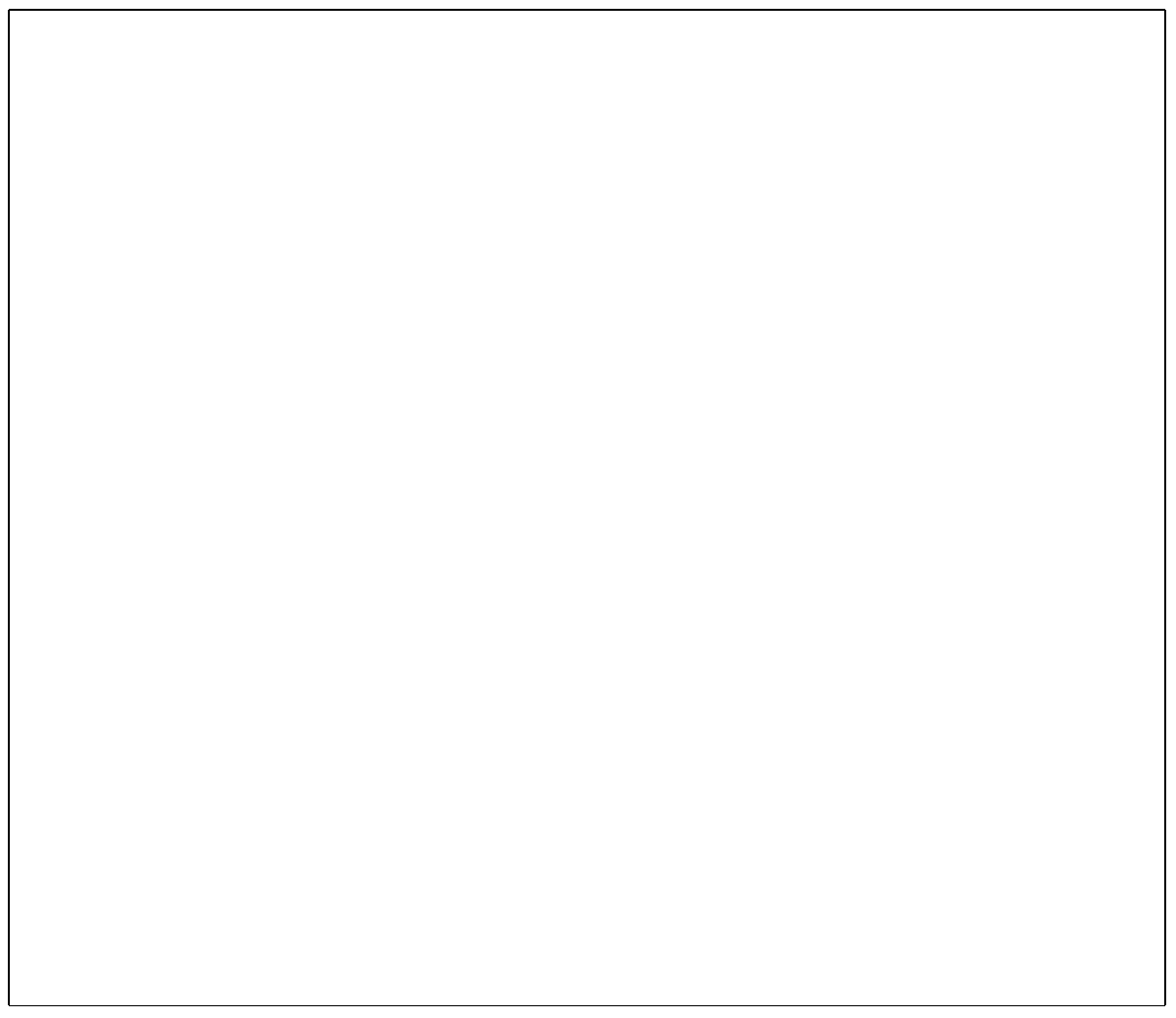}
}
\subfigure[Lowest frequencies only]{
\includegraphics[height=1.8cm,width=1.8cm]{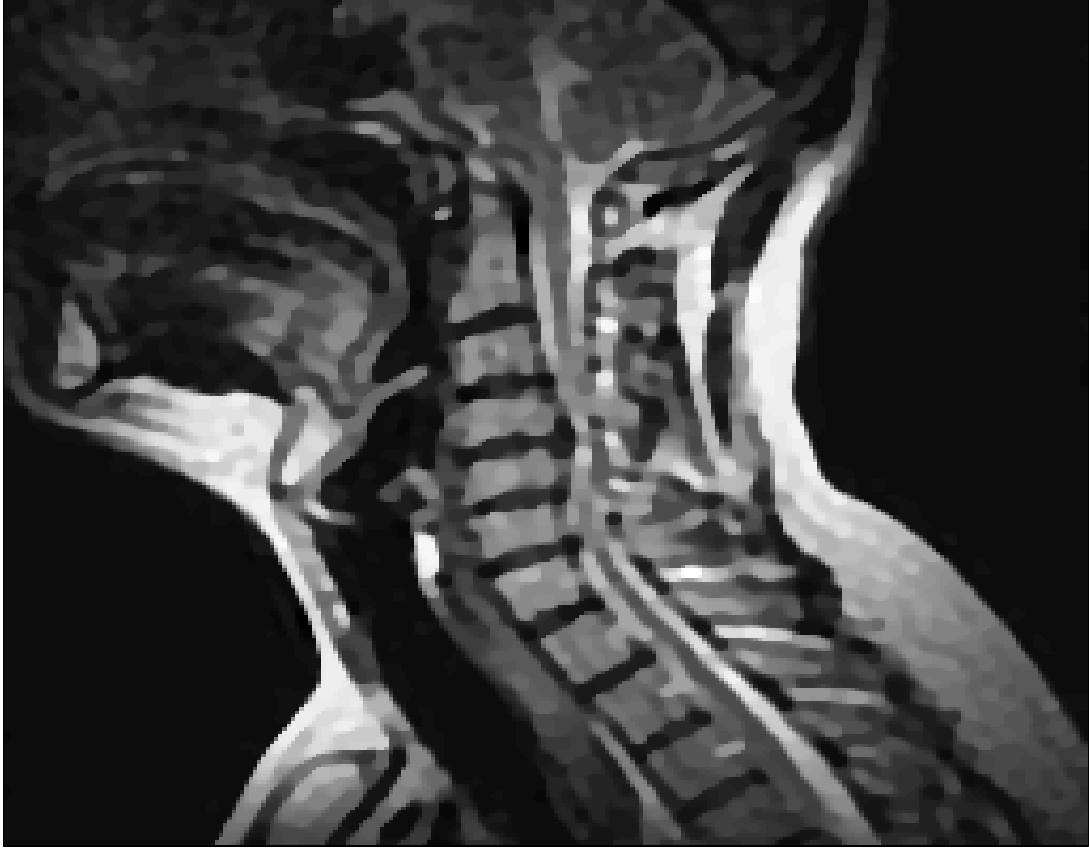} \label{1:z}
\includegraphics[height=1.8cm,width=1.8cm]{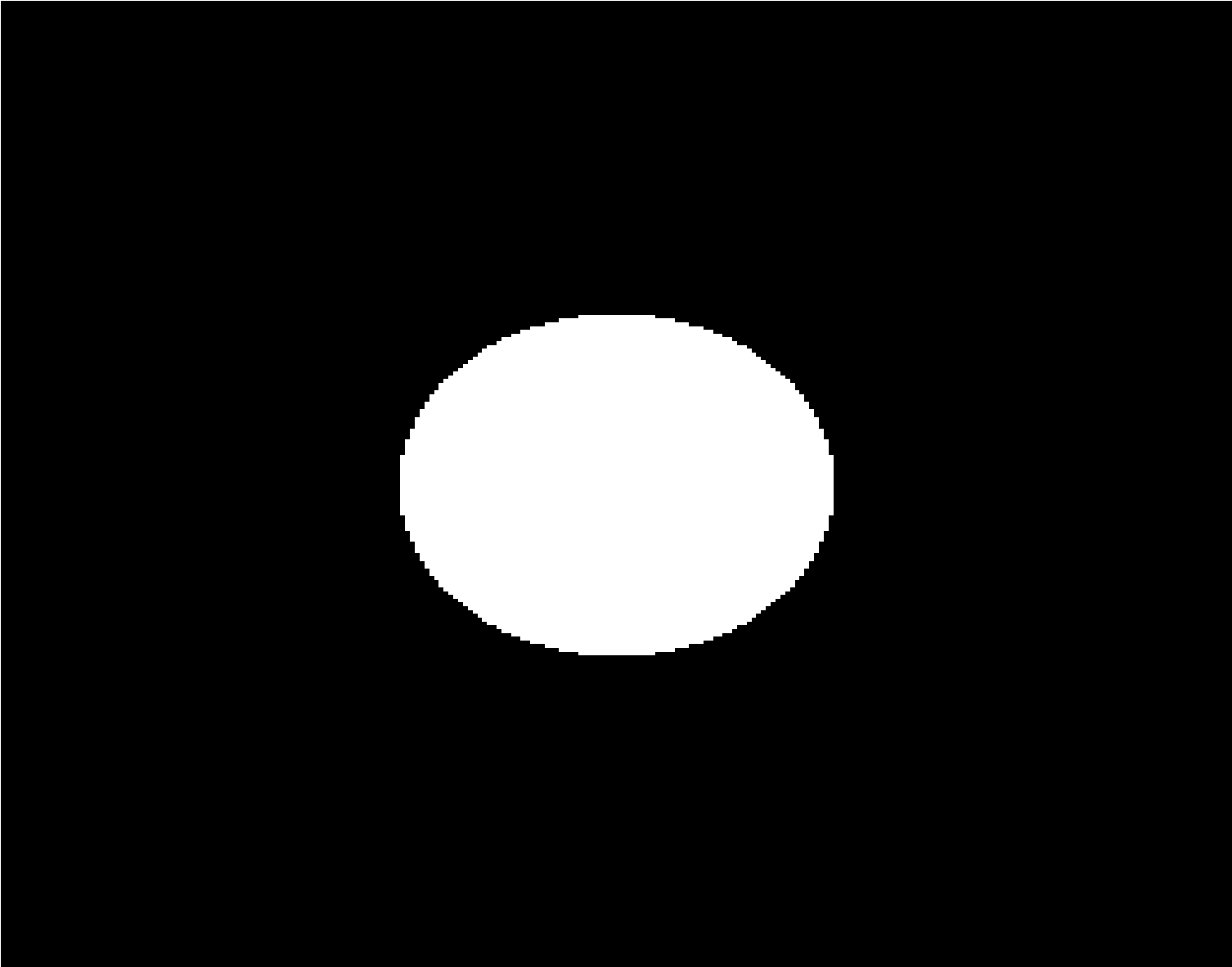} \label{2:z}
}
}
\mbox{
\subfigure[Uniform subsampling]{
\includegraphics[height=1.8cm,width=1.8cm]{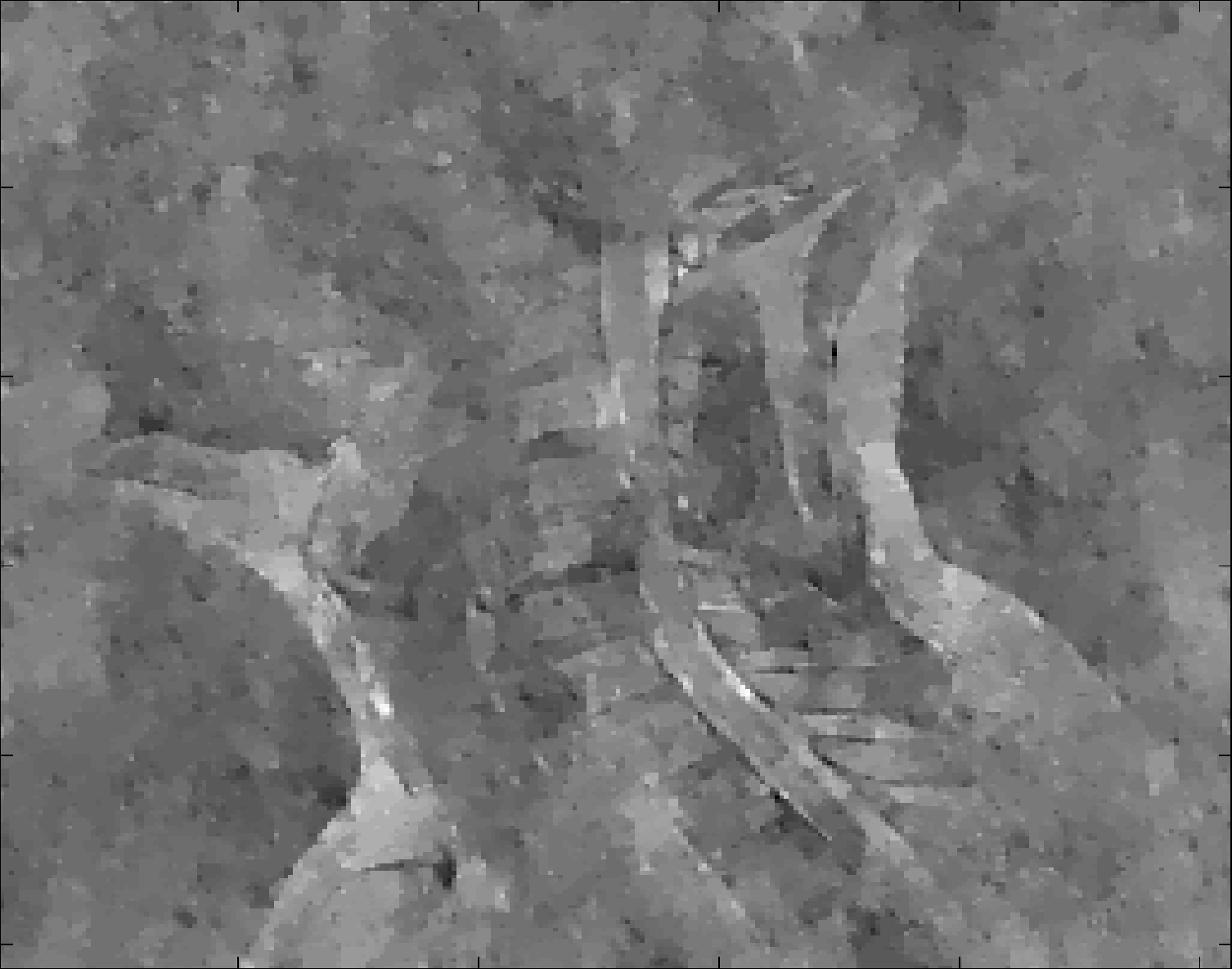} \label{2:a}
\includegraphics[height=1.8cm,width=1.8cm]{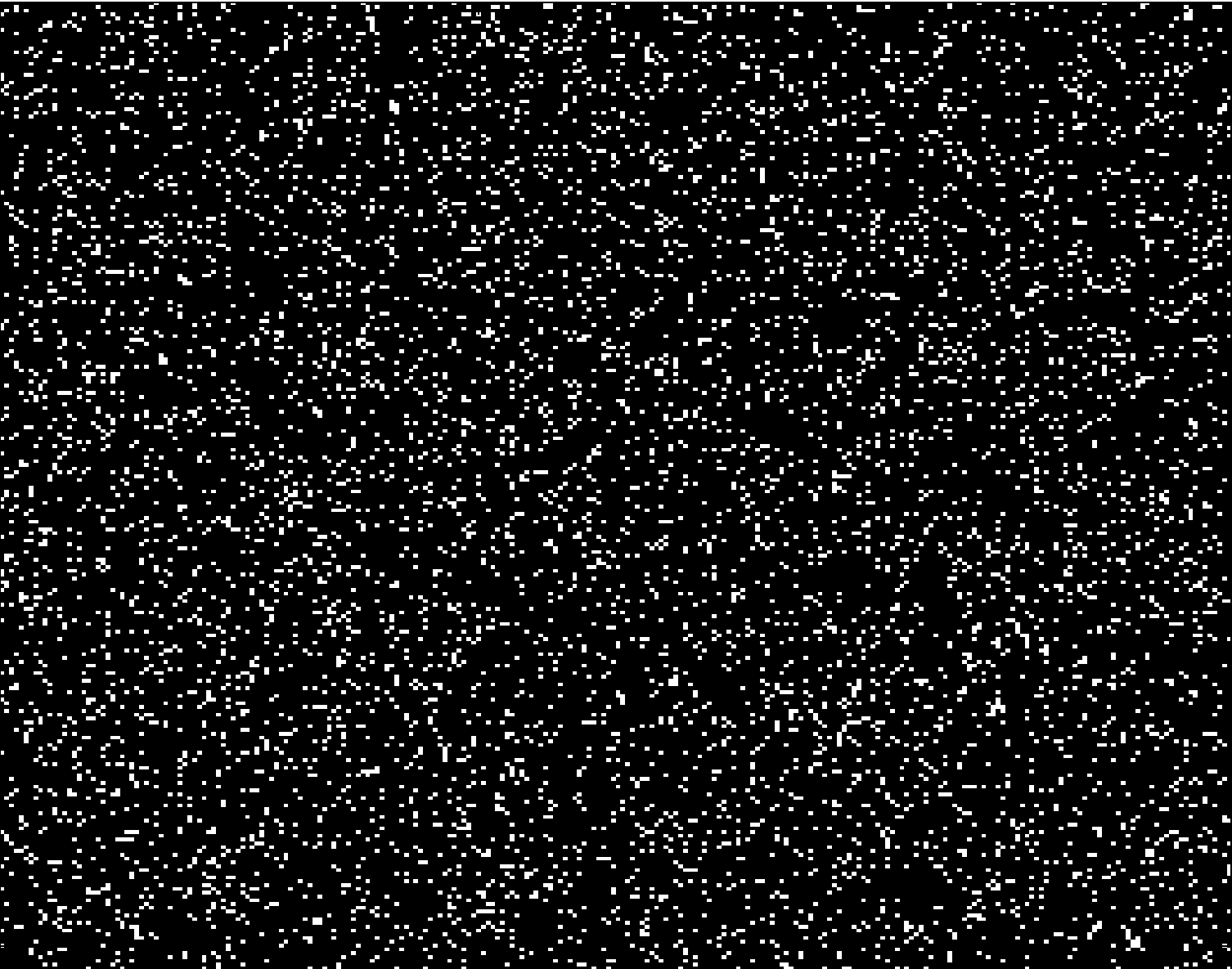} \label{1:a}
}
\subfigure[Equispaced radial lines]{
\includegraphics[height=1.8cm,width=1.8cm]{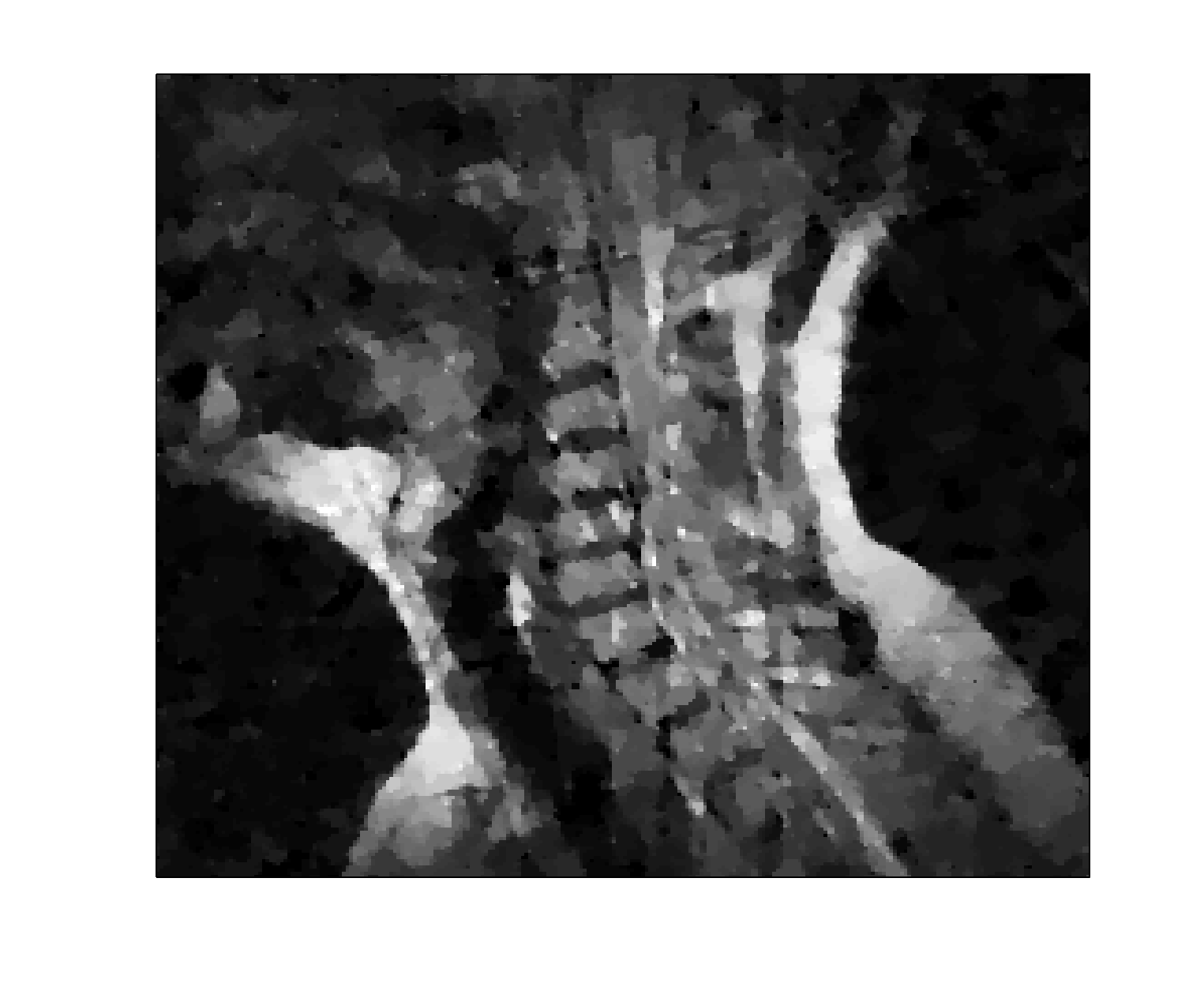} \label{2:b}
\includegraphics[height=1.8cm,width=1.8cm]{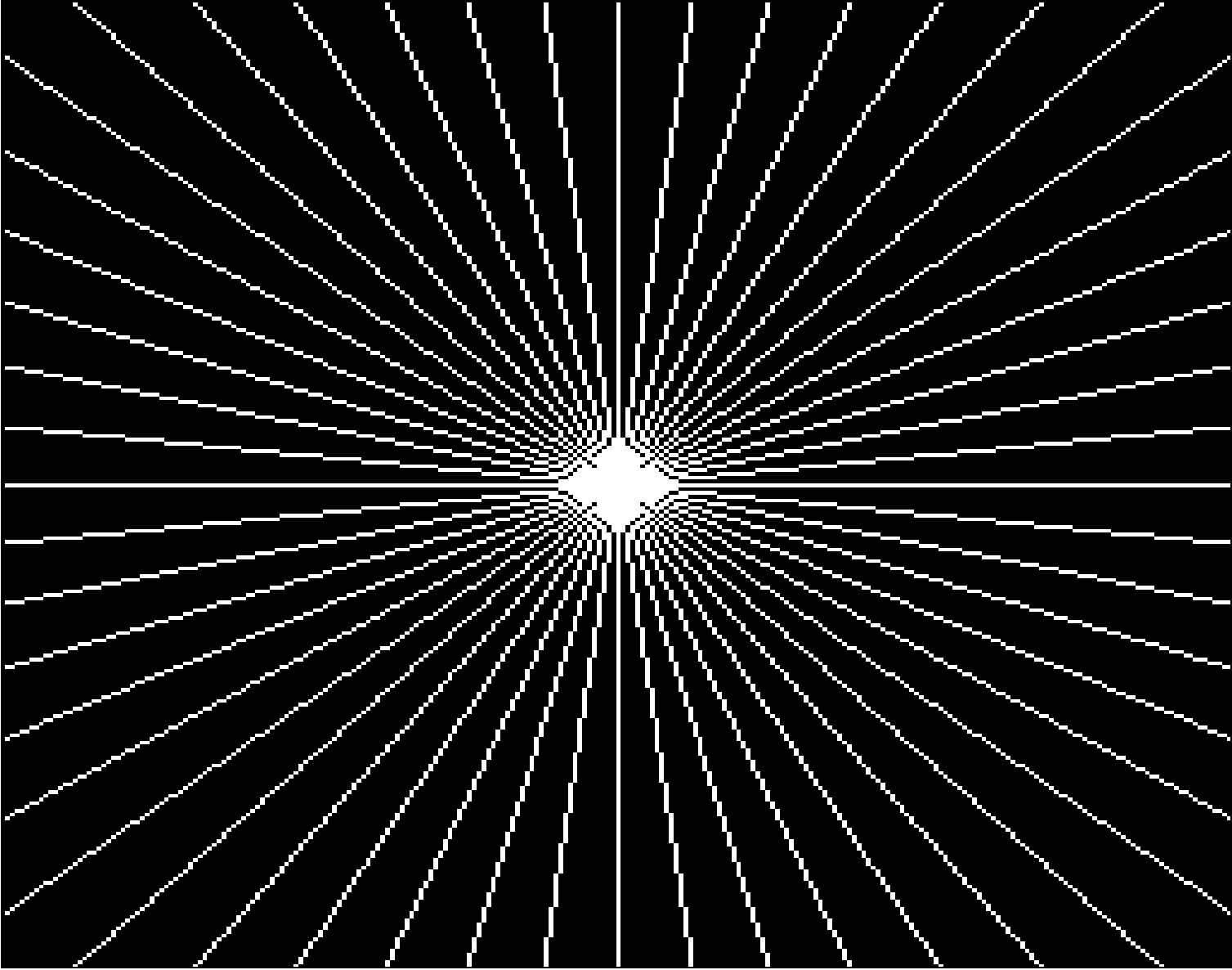} \label{1:b}
}
}
\mbox{

\subfigure[\quad Sample $ \propto (k_1^2 + k_2^2)^{-1/2}$]{
\includegraphics[height=1.8cm,width=1.8cm]{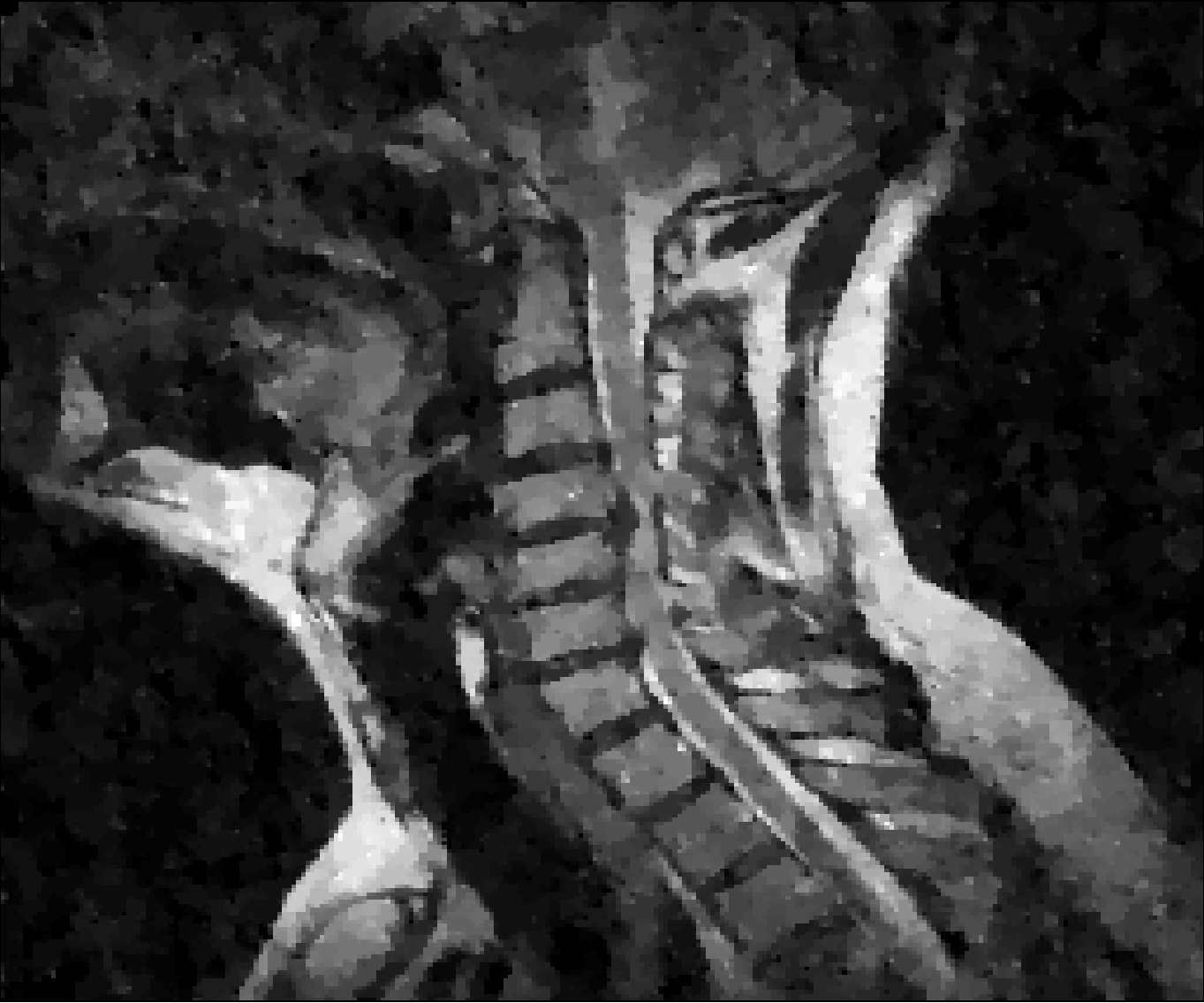} \label{2:c}
\includegraphics[height=1.8cm,width=1.8cm]{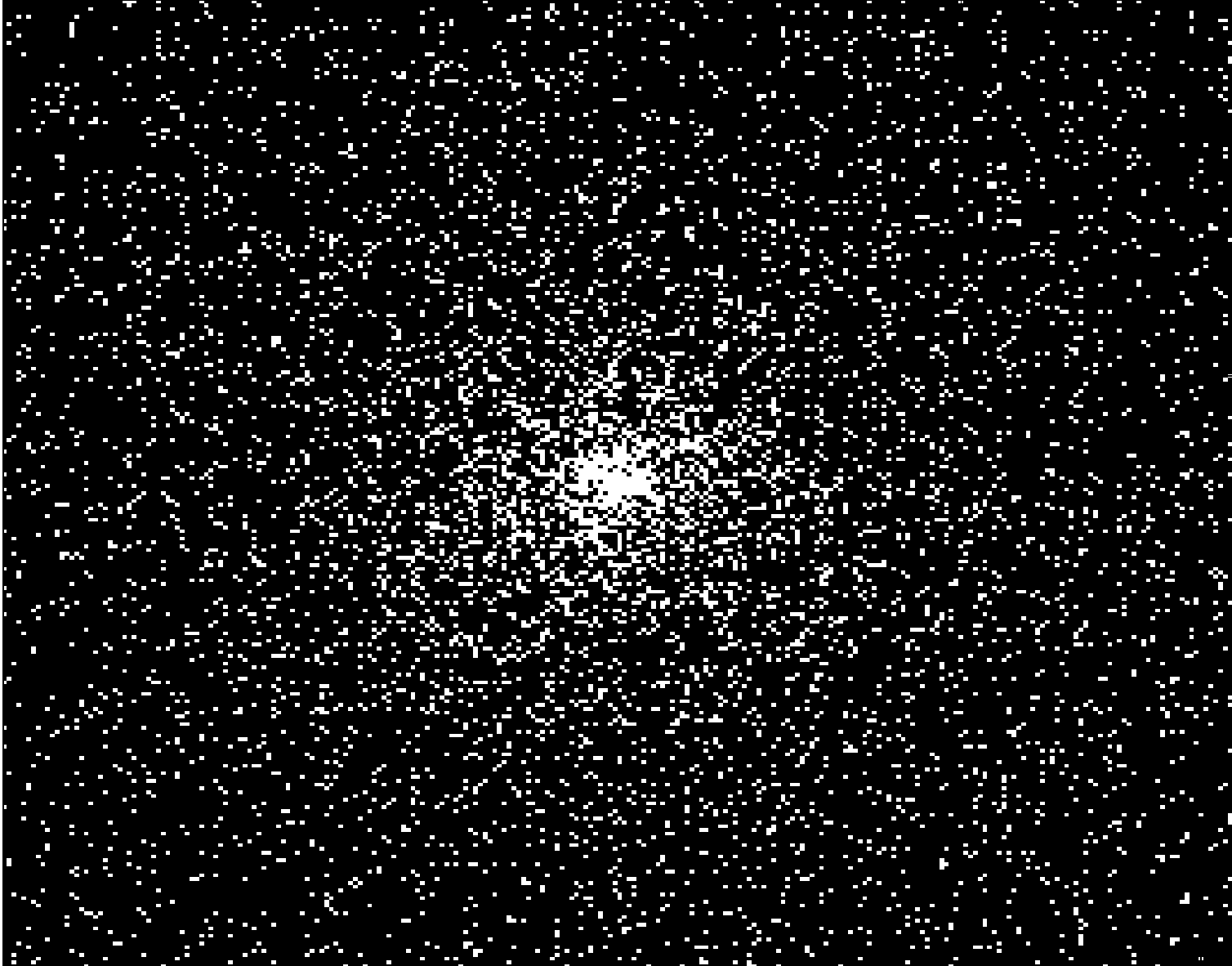} \label{1:c}
}
\subfigure[Sample $\propto \max(|k_1|, |k_2| )^{-1}$]{
\includegraphics[height=1.8cm,width=1.8cm]{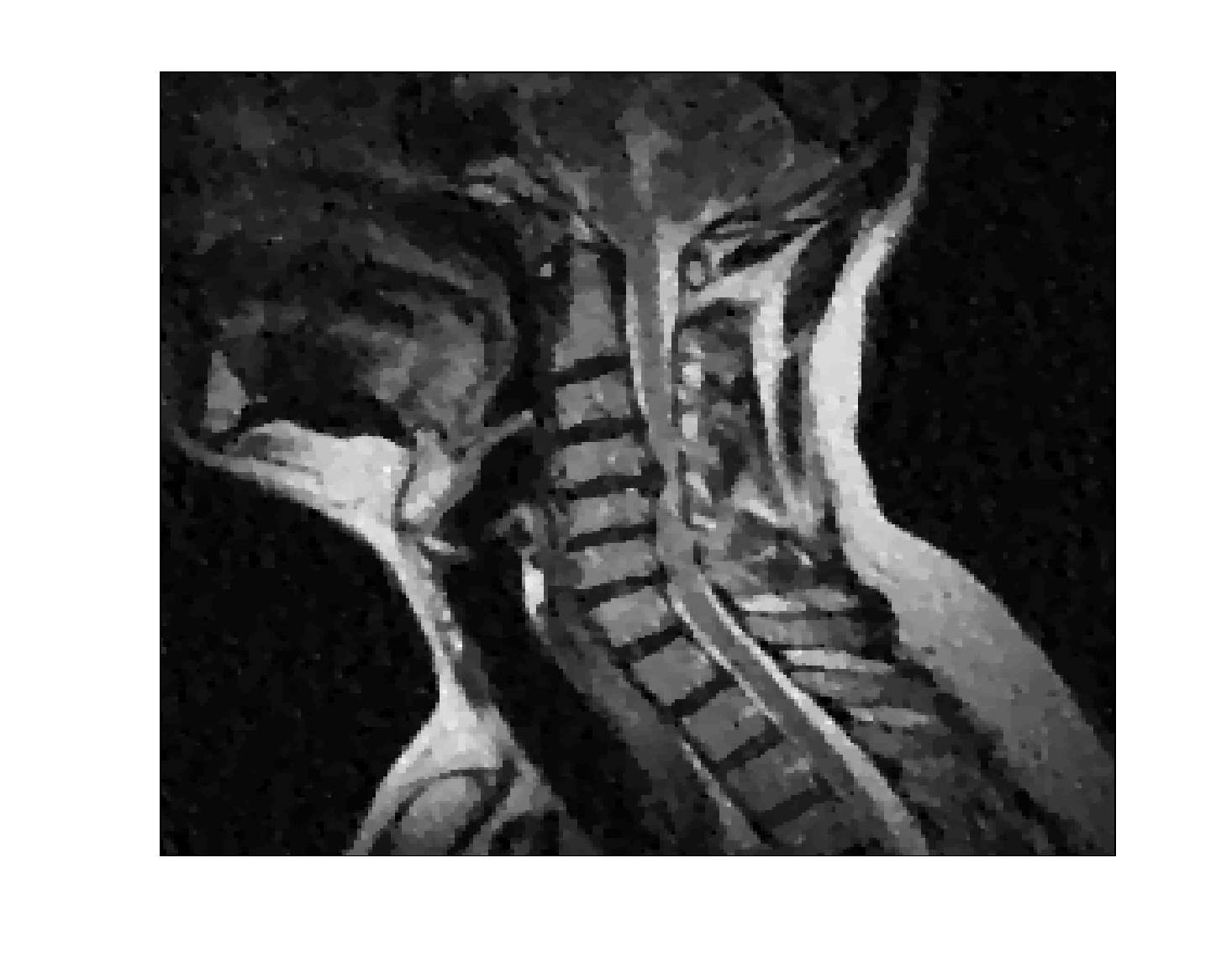} \label{2:d}
\includegraphics[height=1.8cm,width=1.8cm]{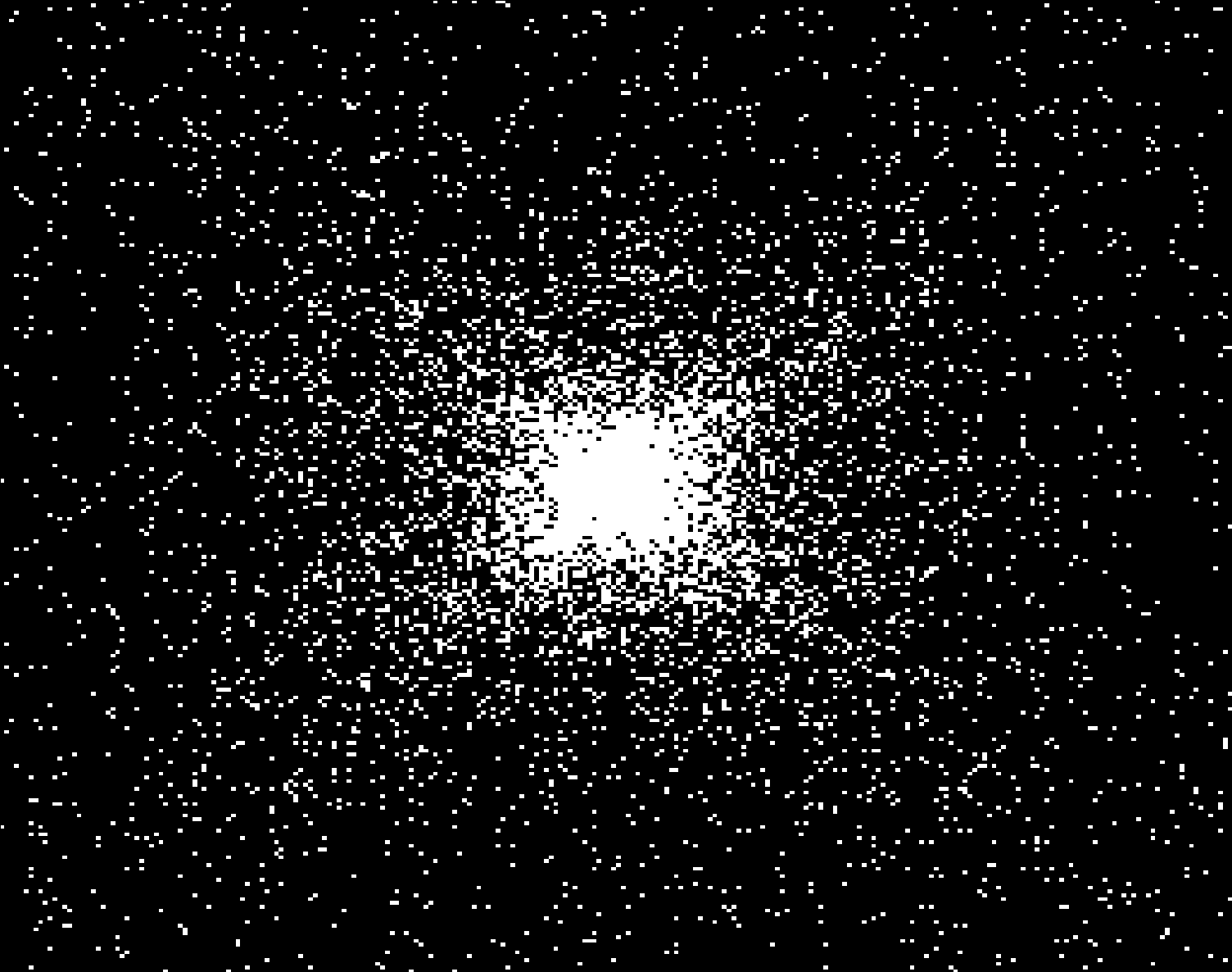} \label{1:d}
}
}

\mbox{
\subfigure[Sample $\propto (k_1^2 + k_2^2)^{-1}$]{
\includegraphics[height=1.8cm,width=1.8cm]{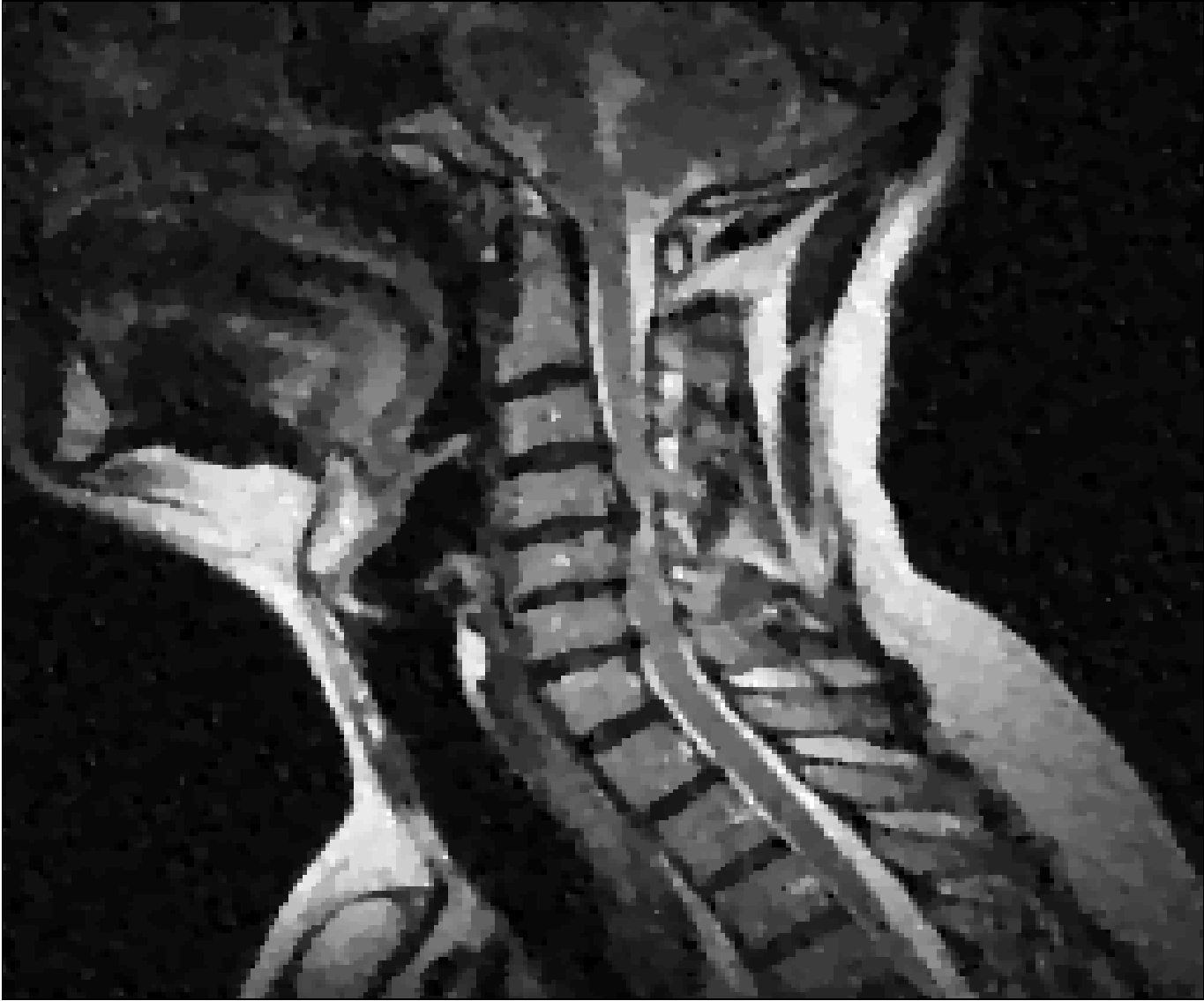} \label{2:e}
\includegraphics[height=1.8cm,width=1.8cm]{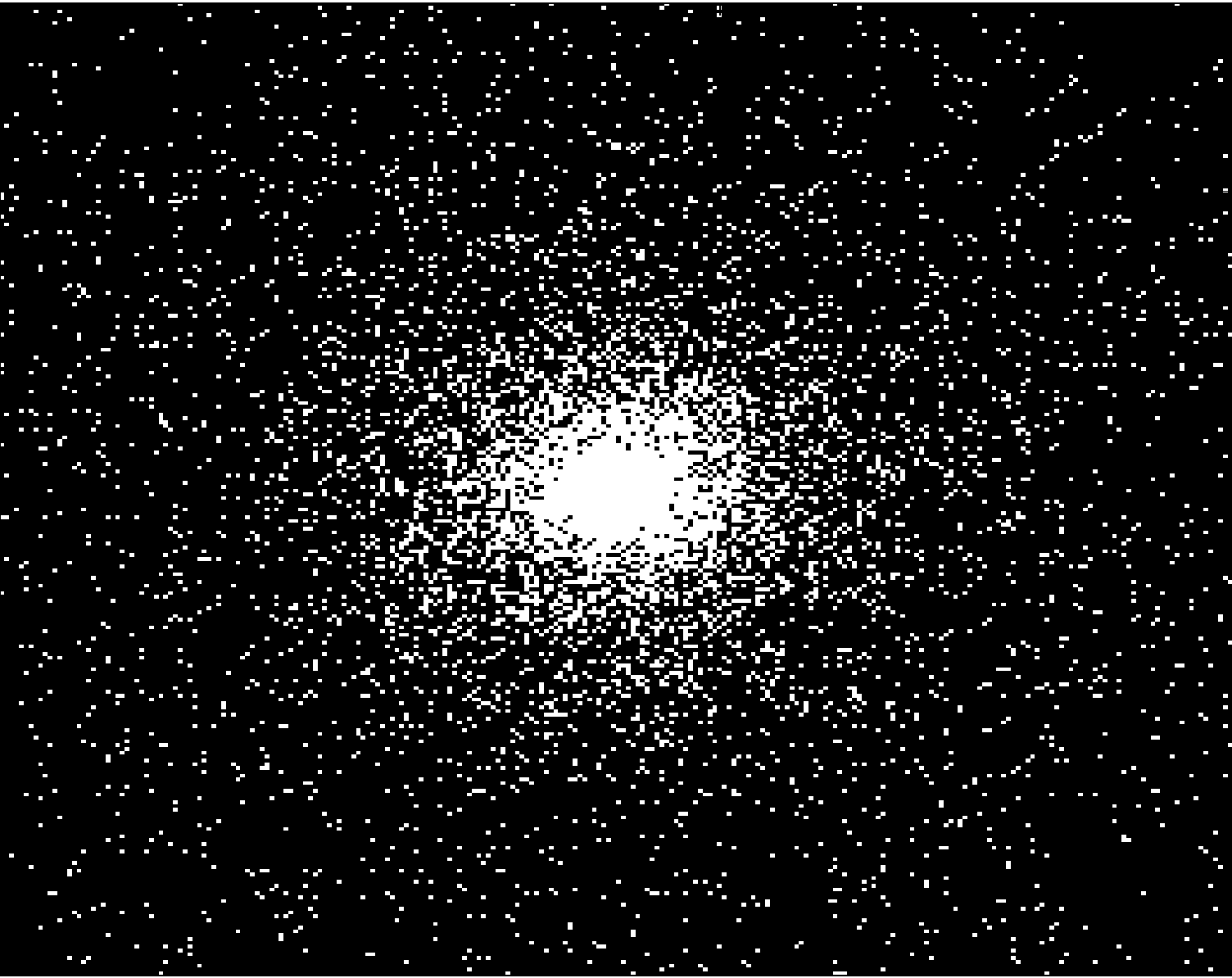} \label{1:e}
}

\subfigure[Sample $\propto (k_1^2 + k_2^2)^{-3/2}$]{
\includegraphics[height=1.8cm,width=1.8cm]{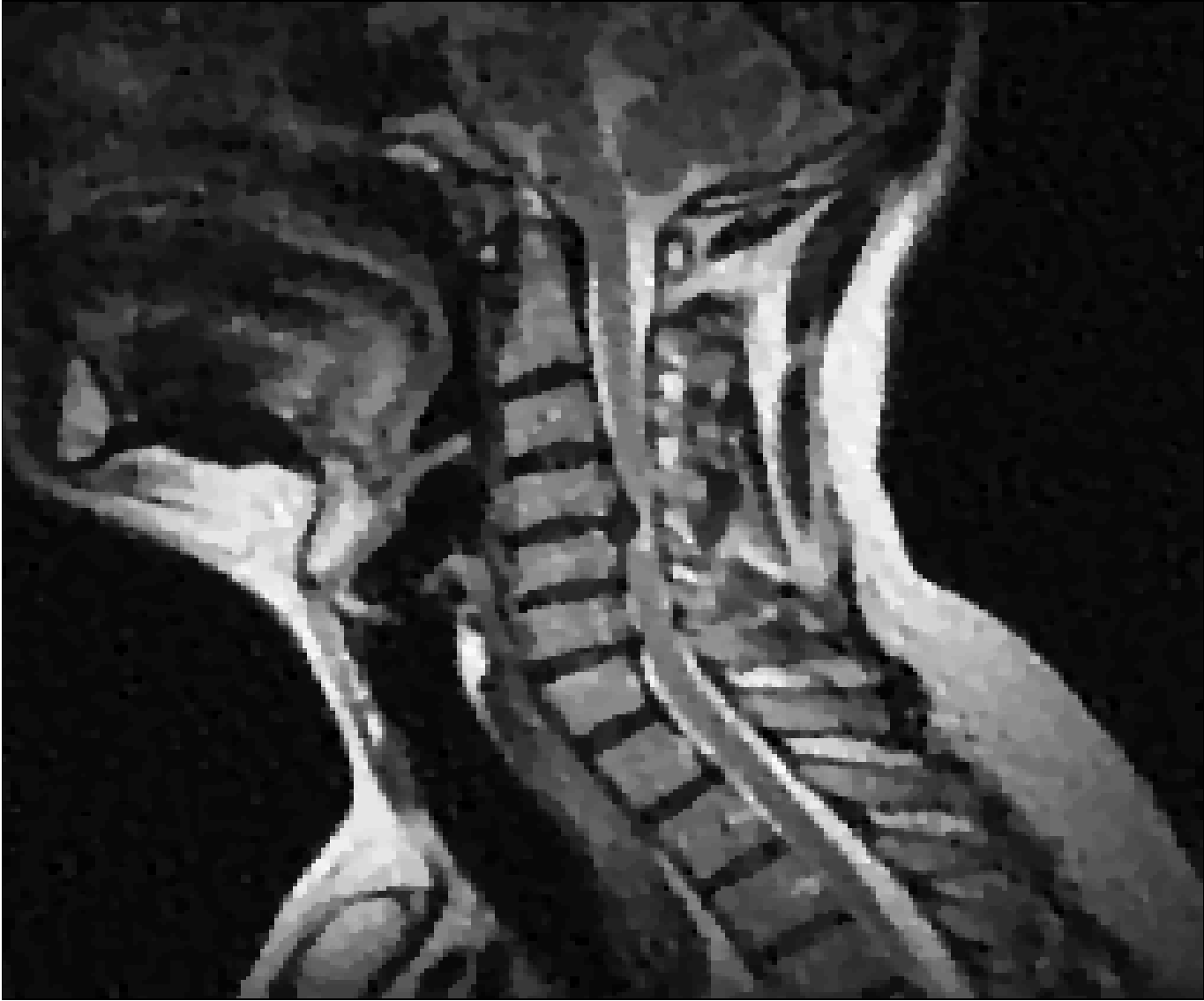} \label{2:f}
\includegraphics[height=1.8cm,width=1.8cm]{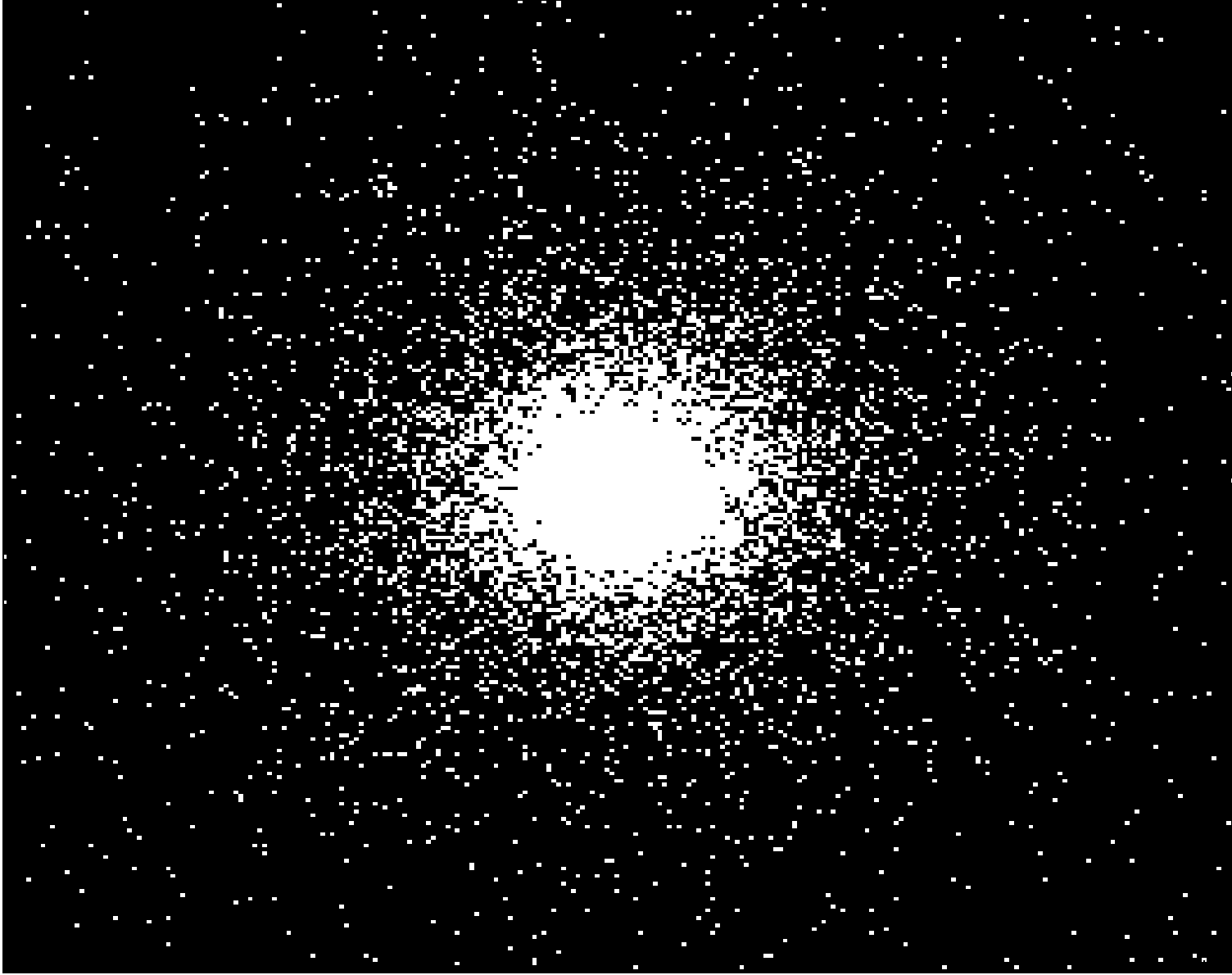} \label{1:f}
}
}
\end{center}
%\subfigure[]
%{\includegraphics[scale =.52] {reltime10.pdf} \label{1:c} 
%}
\caption{\label{fig:1} {\footnotesize Various reconstructions of a 256 by 256 MRI image with total variation minimization as in Theorem \ref{thm1} with $\varepsilon = .001$ and using $m=6400$ noiseless partial DFT measurements with frequencies $\Omega = (k_1, k_2)$ sampled from various distributions.  Beside each reconstruction is a plot of K-space $\{ (k_1, k_2): -N/2 +1 \leq k_1, k_2 \leq N/2 \}$  and the frequencies used (in white).
Theorem \ref{thm1} guarantees stable and robust recovery  for the inverse square-distance distribution in (g); a slightly stronger guarantee can be obtained for the inverse-max sampling distribution given in (f) from the stronger local coherence bound in 
Theorem~\ref{Bincoherent}.  The $\ell_2$ relative errors of reconstruction corresponding to each are (b) .29,  (c) .82, (d) .41, (e) .32, (f) .26, (g) .25, and (h)  .24.  }}
\medskip
\end{figure} 

For a more detailed comparison at higher resolution, we consider in Figures \ref{fig:2} and \ref{fig:3} the $1024^2$ pixel {\em wet paint} image \cite{riceweb}. In a first experiment, we use the relatively low number of $m=12,000$ samples -- slightly more than $1\%$ of the number of pixels -- and visually compare inverse quadratic, inverse cubic, and low-resolution sampling. Again, the relative $\ell_2$ reconstruction errors are comparable in all three cases.  Visually, however, the variable density reconstructions recover more fine details such as the print on the wet paint sign (see Figure \ref{fig:3}).

\begin{figure}[h!]
\begin{center}
\mbox{
\subfigure[Original image]{
 \includegraphics[height=1.8cm,width=1.8cm]{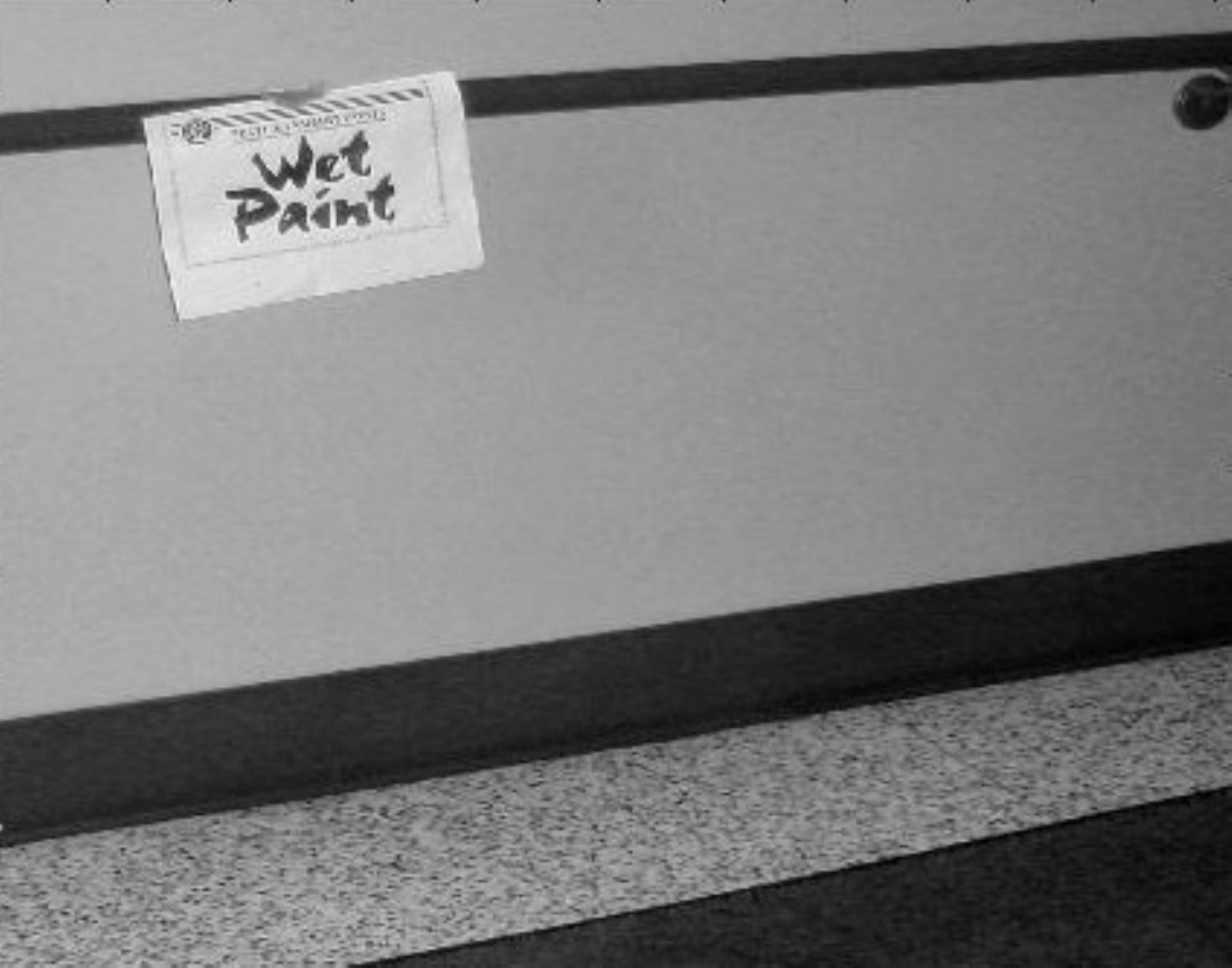}
 \includegraphics[height=1.8cm, width=1.8cm]{full_sample}
}
\subfigure[Lowest frequencies only]{
\includegraphics[height=1.8cm,width=1.8cm]{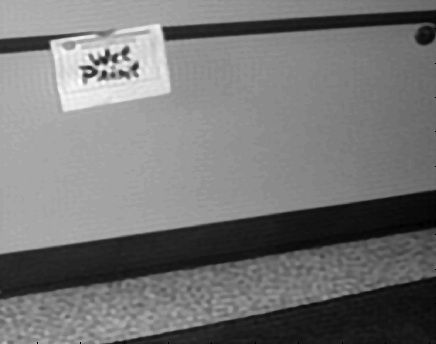} \label{1:zz}
\includegraphics[height=1.8cm,width=1.8cm]{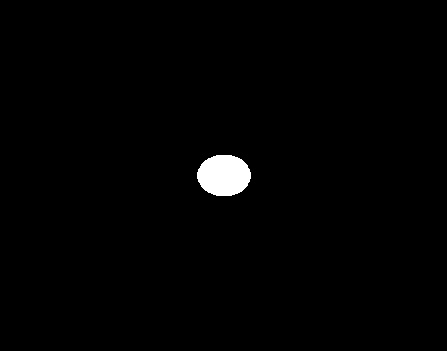} \label{2:zz}
}
}
\mbox{
\subfigure[Sample $\propto (k_1^2 + k_2^2)^{-1}$]{
\includegraphics[height=1.8cm,width=1.8cm]{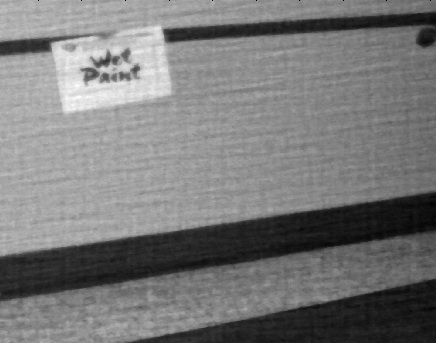} \label{2:aa}
\includegraphics[height=1.8cm,width=1.8cm]{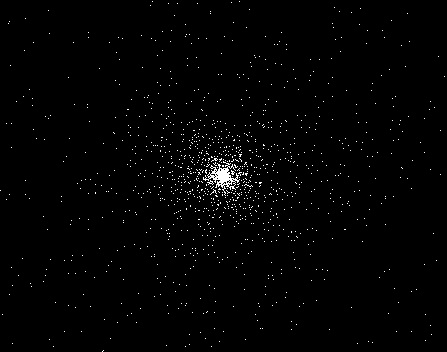} \label{1:aa}
}
\subfigure[Sample $\propto (k_1^2 + k_2^2)^{-3/2}$]{
\includegraphics[height=1.8cm,width=1.8cm]{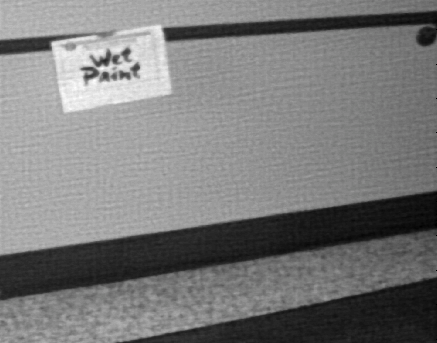} \label{2:bb}
\includegraphics[height=1.8cm,width=1.8cm]{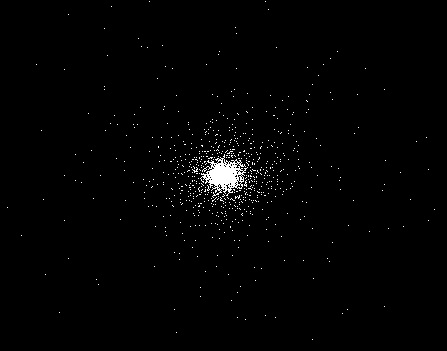} \label{1:bb}
}
}
\end{center}
\caption{\label{fig:2} {\footnotesize Various reconstructions of a $1024^2$ pixel \emph{wet paint} image with total variation minimization as in Theorem \ref{thm1} with $\varepsilon = .001$ and using $m=12,000$ noiseless partial DFT measurements with frequencies $\Omega = (k_1, k_2)$ sampled from various distributions.  Beside each reconstruction is a plot of K-space $\{ (k_1, k_2): -N/2 +1 \leq k_1, k_2 \leq N/2 \}$  and the frequencies used (in white). The relative reconstruction errors corresponding to each  reconstruction are (b) .18,  (c) .21, and (d) .19}}
\medskip
\end{figure} 

\begin{figure}[h!]
\begin{center}
\mbox{
\subfigure[Original image]{
 \includegraphics[height=1.8cm,width=1.8cm]{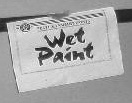}
 \includegraphics[height=1.8cm, width=1.8cm]{full_sample}
}
\subfigure[Lowest frequencies only]{
\includegraphics[height=1.8cm,width=1.8cm]{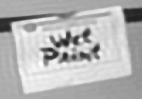} 
\includegraphics[height=1.8cm,width=1.8cm]{DetermKspace}
}
}
\mbox{
\subfigure[Sample $\propto (k_1^2 + k_2^2)^{-1}$]{
\includegraphics[height=1.8cm,width=1.8cm]{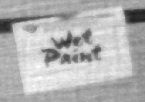}
\includegraphics[height=1.8cm,width=1.8cm]{Pow2Kspace}
}
\subfigure[Sample $\propto (k_1^2 + k_2^2)^{-3/2}$]{
\includegraphics[height=1.8cm,width=1.8cm]{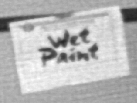}
\includegraphics[height=1.8cm,width=1.8cm]{Pow3Kspace}
}
}
\end{center}
\caption{\label{fig:3} {\footnotesize The reconstructions from Figure \ref{fig:2}, zoomed in.}}
\medskip
\end{figure} 

In a final experiment, we still use the 1024 by 1024 wet paint image, but we now add i.i.d. Gaussian noise to the measurements and compare the reconstruction error for different power law densities. At low SNR ($\ell_2$ norm of signal is 10 times as large as $\ell_2$ norm of noise), the reconstruction error is again comparable for all powers except for uniform sampling (power 0).  At a higher noise level (signal $\ell_2$ norm is only twice as large as noise $\ell_2$ norm), uniform sampling completely fails to recover the image, returning an almost constant image, while the power law densities and low frequency sampling return comparable relative $\ell_2$-norm errors around 0.6.  Despite the comparable $\ell_2$-norm errors,  we still observe visually that power-law sampling is able to recover fine details of the image better than low-frequency sampling, as exemplified by the comparison between the reconstructions for the inverse-quadratic density from Theorem \ref{thm1} and the low frequency-only scheme, also 
plotted in Figure \ref{fig:4}.

\begin{figure}[h!]
\begin{center}
\subfigure[Reconstruction errors by various power-law density sampling at low noise (filled line) and high noise (dashed line)]{
 \includegraphics[height=2.5cm,width=7cm]{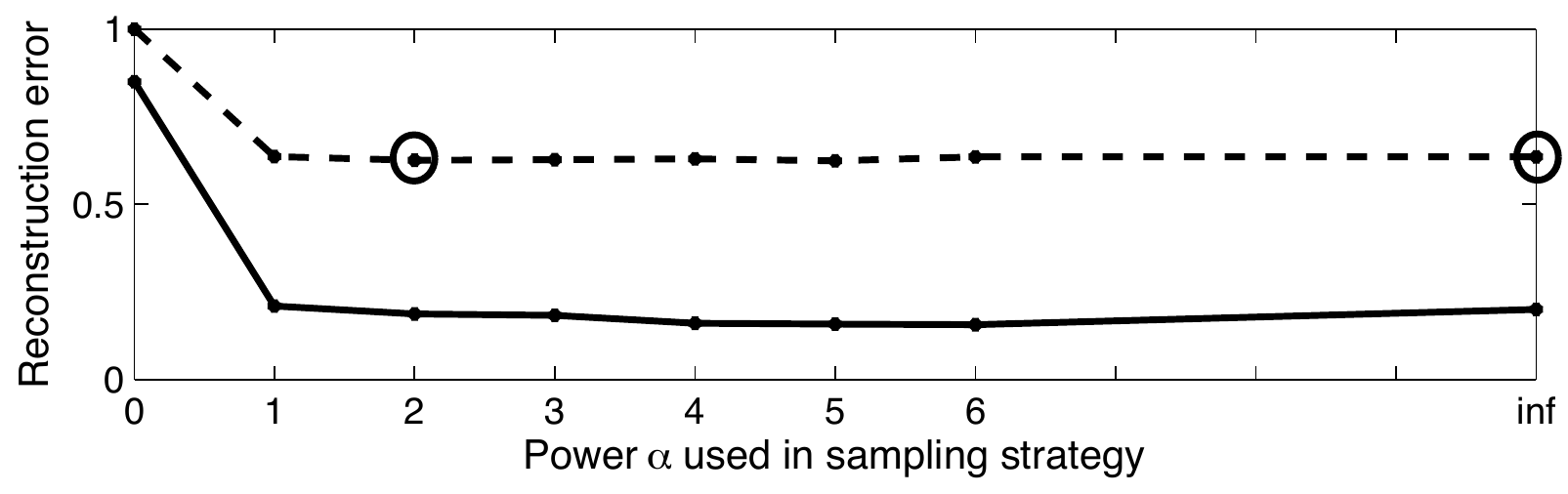}
 }
\mbox{
\subfigure[The \emph{wet paint} reconstructions indicated by the circled errors on the error plot, zoomed in on the paint sign. At high noise level, inverse quadratic-law sampling (left) still reconstructs fine details of the image better than low frequency-only sampling (right).]{
\includegraphics[height=3cm,width=3cm]{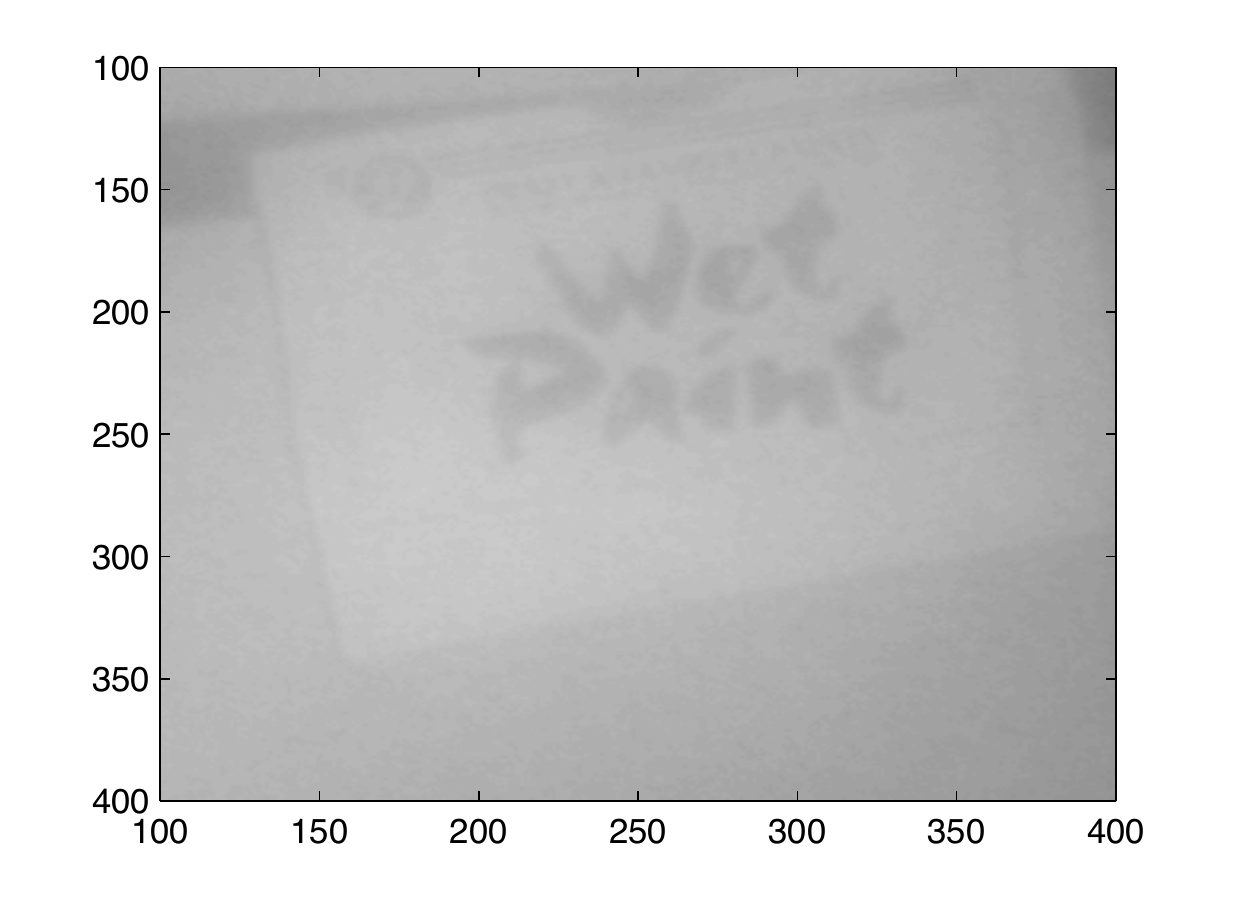} \quad \quad 
\includegraphics[height=3cm,width=3cm]{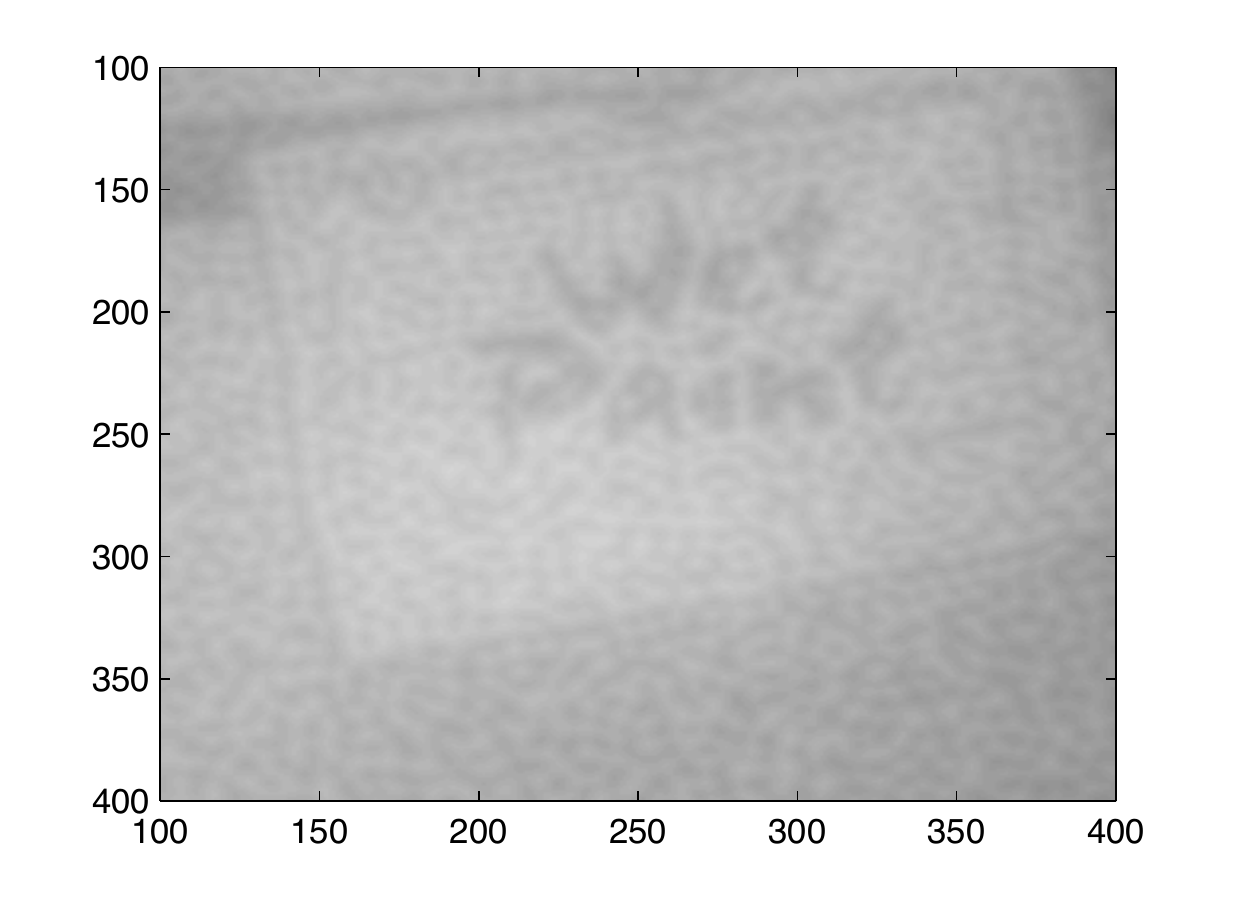}
}}

\end{center}
\caption{\label{fig:4} {The plot in \footnotesize (a) $\ell_2$ relative errors incurred by reconstructing the $1024^2$-pixel \emph{wet paint} image $x$  with total variation minimization as in Theorem \ref{thm1} with normalization $\| x \|_2 = 1$ and noise levels $\varepsilon = .1$ (filled line) and $\varepsilon = .5$ (dashed line) from $m=50,000$ partial DFT measurements with frequencies $\Omega = (k_1, k_2)$ sampled from power-law densities of the form $\text{Prob}\big( (k_1, k_2) \in \Omega \big) \propto (k_1^2 + k_2^2 + 1)^{-\alpha/2},$ with powers ranging from $\alpha = 0$ (uniform sampling) to $\alpha = 6$, as well as $\alpha = \infty$ (lowest frequencies only). The guarantees from Theorem \ref{thm1} hold only for power $\alpha = 2$.}  In (b), we display the \emph{wet paint} images reconstructed using quadratic power-law sampling and low-frequency only sampling  (power = infinity) which are indicated by circles on the corresponding error plot.}
\medskip
\end{figure} 

One should remark that all of these experiments were performed without preconditioning in the regularization term, while our results contain such a step. Preliminary experiments suggest that this may be an artifact of the proof, and for this reason, our experiments were carried out with the standard noise model in the reconstruction procedure.  A more in depth comparison of various noise models and weighting in the reconstruction poses an interesting object of study for future work. Note that weighted noise models similar to the one resulting from our analysis have been explored in \cite{laurent1, laurent2}.

{\section{Summary and outlook}\label{summary}}
We established reconstruction guarantees for variable-density discrete Fourier measurements in both the wavelet sparsity and gradient sparsity setup. Our results build on local coherence estimates between Fourier and wavelet bases.  Although we derive local coherence estimates only for 1D and 2D Fourier/wavelet systems, such estimates can be extended to higher dimensions by induction, using the tensor-product structure of these bases.

Variable density sampling in compressive imaging has often been justified as taking into account the tree-like sparsity structure of natural images in wavelet bases (e.g., in \cite{WA10}).  We note that our theory does not directly take such signal statistics into account, and depends only on the local incoherence between Fourier and wavelet bases.  Incorporating this additional structure to derive stronger reconstruction guarantees, by either improved sampling strategies or improved reconstruction strategies, remains an interesting and important direction of future research.

{All the recovery guarantees in this paper are uniform, that is, we seek measurement ensembles which allow for approximate reconstruction of all images. For non-uniform recovery guarantees, we expect that the number of measurements required in our main results can be reduced by several logarithmic factors by following a probabilistic and ``RIP-less"  approach \cite{candesplan11}. }

It should also be noted that this paper does not address the important issue of errors arising from discretization of the image and Fourier measurements. In particular, as observed for example in \cite{AH11}, the use of discrete rather than continuous Fourier representations can be a significant source of error in compressive sensing. The authors of \cite{AH11} propose to resolve this issue using uneven sections, that is, the number of discretization points in frequency is chosen to be larger than the number of discretization points in time.  Nevertheless, the results in \cite{AH11} are again just formulated for incoherent samples. 
Recently, it has been proposed to overcome this issue by sampling all of the low frequencies in addition to uniformly sampling the higher frequencies \cite{AHHT11}. After the submission of this paper, reconstruction guarantees for such a setup were provided in \cite{ahpr13}, also for a generalization to multilevel sampling schemes. In addition to an asymptotic notion of coherence (related to the local coherence we look at in this paper), \cite{AH11} also considers an asymptotic notion of sparsity, which relates to the additional structure of wavelet expansions mentioned above. 

We think that it should be an interesting to study how our approach can be applied to infinite dimensional image models -- due to the variable density, it may even be possible to sample from all of the infinite set rather than restricting to a finite 
subset. Such a generalization would prove challenging for the optimization-based approaches such as in \cite{pvw11}, which will always be specific to the given problem dimension.  In this sense, we expect that the additional understanding provided by this paper can eventually lead to optimized sampling schemes. All these questions, however, are left for future work.

\section*{Acknowledgments}
The authors would like to thank Ben Adcock, Anders Hansen, Deanna Needell, Holger Rauhut, Justin Romberg, Amit Singer, Mark Tygert, Robert Vanderbei, Yves Wiaux, and the anonymous reviewers for helpful comments and suggestions.
They are grateful for the stimulating research environment of the Mathematisches Forschungsinstitut Oberwolfach, where part of this work was completed. Rachel Ward was supported in part by an Alfred P Sloan Research Fellowship, a Donald D. Harrington Faculty Fellowship, an NSF CAREER grant, and DOD-Navy grant N00014-12-1-0743.

\bibliographystyle{plain}
\bibliography{FTV}

\end{document}